\documentclass[]{fairmeta}
\newcommand{\remove}[1]{}
\usepackage{microtype}
\usepackage{graphicx}
\usepackage{subcaption}
\usepackage{booktabs}
\usepackage{multirow}
\usepackage{wrapfig}
\usepackage{makecell}
\usepackage{rotating}
\usepackage{amsmath}
\usepackage{amssymb}
\usepackage{mathtools}
\usepackage{amsthm}
\theoremstyle{plain}
\newtheorem{theorem}{Theorem}[section]

\theoremstyle{definition}

\theoremstyle{remark}

\newcommand{\shortmethod}[1]{\textsc{BFF}}
\newcommand{\method}[1]{\textsc{GraphBFF}}
\usepackage{hyperref}
\usepackage{cleveref}

\usepackage[dvipsnames,table]{xcolor}
\colorlet{blue}{RoyalBlue}
\colorlet{green}{ForestGreen}
\colorlet{red}{BrickRed}

\title{Billion-Scale Graph Foundation Models}

\author[1]{Maya Bechler-Speicher}
\author[2]{Yoel Gottlieb}
\author[1]{Andrey Isakov}
\author[1]{David Abensur}
\author[1]{Ami Tavory}
\author[1]{Daniel Haimovich}
\author[1]{Ido Guy}
\author[1]{Udi Weinsberg}

\affiliation[1]{ Meta}
\affiliation[2]{Work done at Meta}

\abstract{
Graph-structured data underpins many critical applications. While foundation models have transformed language and vision via large-scale pretraining and lightweight adaptation, extending this paradigm to general, real-world graphs is challenging.
In this work, we present Graph Billion-Foundation-Fusion (\method{}): an end-to-end recipe for building billion-parameter Graph Foundation Models (GFMs) for large-scale heterogeneous graphs. Central to the recipe is the \method{} Transformer, a flexible and scalable architecture designed for practical billion-scale GFMs. Using the \method{}, we present neural scaling laws for heterogeneous graphs and show that loss decreases predictably as either model capacity or training data scales, depending on which factor is the bottleneck.
The \method{} framework provides concrete methodologies for data batching, pretraining, and fine-tuning for building GFMs at scale. We demonstrate the effectiveness of the framework over a real-world billion-scale graph, with an evaluation of a billion-parameter \method{} Transformer following the proposed recipe. Across ten diverse, real-world downstream tasks on graphs unseen during training, spanning node- and link-level classification and regression, \method{} consistently outperforms baselines, with large margins of up to $31$ PRAUC points, including in few-shot settings. Finally, we discuss key challenges and open opportunities for making GFMs a practical and principled foundation for graph learning at industrial scale.
}

\correspondence{Maya Bechler-Speicher \email{mayabs@meta.com}}

\begin{document}

\maketitle

\section{Introduction}

\begin{figure}[t]
        \centering
     \includegraphics[width=0.8\linewidth]{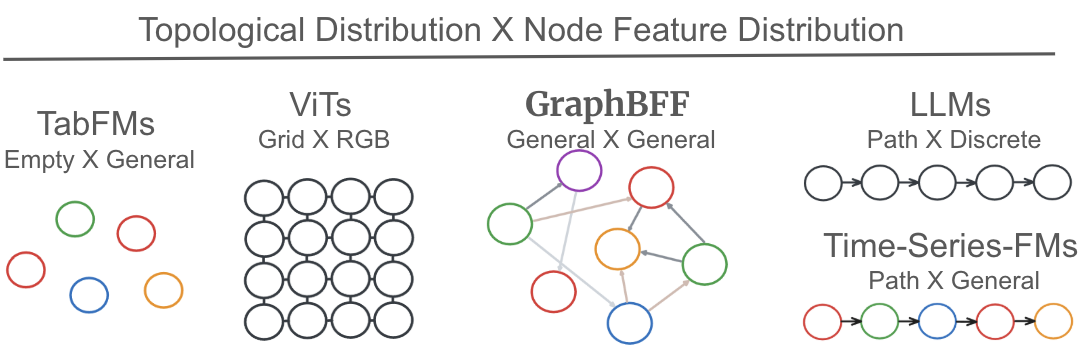}
    \caption{An illustration of FMs as GFMs over specific topological and feature distributions, using the minimal graph structure they operate on. Node colors correspond to token types, with each type undergoing distinct transformations within the model. The \textsc{GraphBFF} is designed for general graphs and feature distributions, supporting any number of token types.}
    \label{fig:gfms}
\end{figure}

Graph-structured data is ubiquitous across domains such as security, social networks, recommender systems, and many others. While foundation models have revolutionized natural language processing and computer vision through large-scale pretraining and lightweight adaptation to downstream tasks  \citep{bommasani2021foundation}, extending these advances to graphs in the form of Graph Foundation Models (GFMs) is fundamentally challenging.
First, graph-data distributions differ substantially across domains, for example, molecular graphs and social networks vary in node-feature distributions, topological structure, and scale. This raises a fundamental question about the validity of pretraining, as there might be little transferable structure across graphs. Furthermore, the scarcity of public, high-quality, large-scale graph datasets \citep{bechlerspeicher2025positiongraphlearninglose} limits the ability to rigorously study GFMs at billion-node data scales and billion-parameter model sizes. Finally, graph learning problems span multiple levels of granularity, including node-, edge-, and graph-level tasks.

While a truly generic GFM is challenging, many data modalities, including text and images, can be viewed as instances of graph-structured data with characteristic graph distributions \citep{Veli_kovi__2023,bronstein2021geometricdeeplearninggrids}. Under this perspective, Large Language Models (LLMs) \citep{vaswani2023attentionneed,brown2020language,devlin2019bert} and Vision Transformers (ViTs) \citep{dosovitskiy2021an} can be interpreted as billion-scale GFMs. 
LLMs assume a sequential token order and exploit it through positional encodings. ViTs assume images lie on a fixed two-dimensional grid, leveraging this structure via patchification and parameter sharing. \Cref{fig:gfms} illustrates FMs as GFMs, using the minimal graph structure they operate on.
However, these architectures embed strong inductive biases tailored to their underlying data distributions, which can hinder performance when applied to general graphs drawn from different distributions. For example, reducing graphs to text sequences as input to LLMs may yield poor performance and unstable predictions  \citep{fatemi2023talklikegraphencoding}. In addition, this approach may require LLMs to have orders of magnitude more parameters to outperform models designed for graphs \citep{ranjan2025relationaltransformerzeroshotfoundation}, such as Graph Neural Networks (GNNs) \citep{gilmer2017neural, kipf2017semisupervisedclassificationgraphconvolutional, hamilton2018inductiverepresentationlearninglarge} and Graph Transformers \citep{dwivedi2021graphtransformer, rampasek2022gps}.
\Cref{fig:gfms} illustrates FMs as GFMs, using the minimal graph structure they operate on.
We expand on the viewpoint of FMs as GFMs in \Cref{sec:preliminaries}.

In this work:
\begin{enumerate}
    \item We present Graph Billion-Foundation-Fusion (\method{}) - an end-to-end framework for building billion-parameter GFMs on large-scale heterogeneous graphs, at real industrial scale and settings. 
\item We introduce the \method{} Transformer, a flexible, scalable and effective architecture for building billion-scale GFMs. By leveraging two heterogeneous attention components and incorporating a sparse softmax, our transformer efficiently supports real-world large-scale heterogeneous graphs. We formally show that the two attention components of \method{} Transformer are necessary for its expressiveness.
\item Using \method{} Transformer, we present neural scaling laws in terms of data and model size. These laws show strict model and data bottlenecks for GFMs, suggesting that model and data must grow together, as previously observed in LLMs.
\item We introduce \textit{KL-Batching} and \textit{Round-Robin Batching}, storage-level and GPU-level strategies for effective pre-training on type-skewed billion-scale graphs.
\item We perform an extensive evaluation of a 1.4 billion-parameter \method{} Transformer pretrained on one billion samples from real-world graph data. We examine ten diverse, real-world industrial downstream tasks over graphs that were unseen during training, spanning node- and link-level classification and regression, and report strong probing performance, including in few-shot settings. We also evaluate the \method{} Transformer as a task-specific model, showing that it consistently outperforms existing task-specific heterogeneous graph transformers.
\end{enumerate}
    
While absent from the public domain, many organizations maintain billion-scale graphs and many industrial applications rely on large-scale graph data. Our goal is to provide a concrete, reproducible blueprint, backed by strong empirical evidence, for building effective GFMs in practice. We discuss numerous new theoretical and methodological questions arising from our work, and outlining key open challenges and promising directions for deploying effective GFMs.

\section{Related Work} \label{sec:related}
\textbf{Graph Foundation Models}
The public graph-data landscape still lacks billion-scale, diverse, high-quality data. This scarcity constrains GFMs research and may partly explain why progress on training billion-parameter GFMs lags behind other domains with abundant public data, such as text and vision \citep{bechlerspeicher2025positiongraphlearninglose}.
Another core challenge in building GFMs, emphasized by recent GFM surveys \citep{wang2025graphfoundationmodelscomprehensive}, is heterogeneity along three axes: (i) heterogeneity, (ii) structural heterogeneity, and (iii) task heterogeneity (e.g., node-, edge-, and graph-level tasks).
A prominent line of works focuses on designing GFMs for specific types of graphs with their specific tasks of interest, such as molecular graphs \citep{shoghi2024moleculesmaterialspretraininglarge} or knowledge  graphs \citep{galkin2024foundationmodelsknowledgegraph}.
Another line of work focuses on feature heterogeneity, a challenge also central to TabFMs \citep{gorishniy2021revisiting, somepalli2021saint, hollmann2022tabpfn, shaw2018self, hu2020heterogeneousgraphtransformer}. Feature heterogeneity asks whether a pre-trained model can be applied to samples with previously unseen feature sets. In practice, this often reduces to a technical mismatch between model parameters and the input space, motivating approaches that restrict inputs to a predefined vocabulary \citep{mao2024positiongraphfoundationmodels, wang2025learninggraphquantizedtokenizers}.
A different line of work extends TabFMs ideas to graphs, by partitioning features into predefined groups (e.g., numerical, categorical, text) and enforcing shared transformations within each group \citep{eremeev2025turningtabularfoundationmodels, finkelshtein2025equivarianceoncerecipegraph, GraphAny2024, ranjan2025relationaltransformerzeroshotfoundation, liu2024alltraininggraphmodel}.
While these strategies ensure dimensional compatibility for unseen node types, they can limit expressivity by forcing semantically distinct features through the same transformation. This exposes an inherent trade-off: grouping is a practitioner-driven design choice, and how to select it remains an open question.
Importantly, our framework is compatible with any of the above methods, and focuses at the billion-scale model parameters and data.

\textbf{Transformers on Graphs}
Applying Transformers to graphs requires choosing how to represent a graph as tokens, for example by converting it into node and edge tokens arranged as a set or sequence \citep{dwivedi2021graphtransformer,kreuzer2021rethinking,zhang2020graphbert,ying2021transformer}. Recent works show that this design choice strongly influences both expressivity and scalability \citep{yehudai2025depthwidthtradeoffsalgorithmicreasoning, sanford2023representationalstrengthslimitationstransformers, sanford2024transformersparallelcomputationlogarithmic}.
To better capture topology, another dominant approach involves constraining or biasing the attention mechanism. By employing graph-informed attention masks or structural biases, models can explicitly inject graph-structure priors into the self-attention process \citep{ying2021transformer,velickovic2018graph,zhang2020graphbert,dwivedi2021graphtransformer,kreuzer2021rethinking, rampasek2022gps}.
Heterogeneous graph Transformers target graphs with multiple node and edge types, requiring relation-aware attention rather than treating all edges uniformly. HGT conditions self-attention on node/edge types via type-specific projections and relation-dependent parameters \citep{hu2020heterogeneousgraphtransformer}. Related heterogeneous attention models leverage schema structure through meta-path–guided aggregation, as in HAN and MAGNN \citep{wang2021heterogeneousgraphattentionnetwork,fu2020magnn}. See \citet{Shehzad_2026} for a recent Survey on Graph Transformers.

\section{Preliminaries} \label{sec:preliminaries}
As in most foundation model settings, we assume a data universe from which both pre-training data and downstream task data are drawn. In the general case, this universe can be represented as a heterogeneous graph
\[
\mathcal{G} = (\mathcal{V}, \mathcal{E}, \tau, \phi)\]
where $\mathcal{V}$ is a set of nodes, $\mathcal{E} \subseteq \mathcal{V} \times \mathcal{V}$ is a set of edges, $\tau : \mathcal{V} \to \mathcal{T}_V$ assigns each node a type from a finite set of node types $\mathcal{T}_V$, and $\phi : \mathcal{E} \to \mathcal{T}_E$ assigns each edge a relation type from a finite set of edge types $\mathcal{T}_E$. 
A \textit{heterogeneous graph} is a graph with $|\mathcal{T}_E| >1$ or  $|\mathcal{T}_V| > 1$.\newline
Each node $v \in \mathcal{V}$ is associated with an input feature vector $\mathbf{x}_v \in \mathcal{X}_{\tau(v)} \subseteq \mathbb{R}^{d_{\tau(v)}}$, where $\mathcal{X}_{\tau(v)}$ denotes the feature support of node type $\tau(v)$. 
Each edge $(u,v) \in \mathcal{E}$ may also have an associated feature vector $\mathbf{x}_{uv} \in \mathcal{X}_{\phi(u,v)} \subseteq \mathbb{R}^{d_{\phi(u,v)}}$. We assume that node features are drawn from a type-specific marginal distribution $\mathbf{x}_v \sim P_{\tau(v)}$ with $\mathrm{supp}(P_{\tau(v)}) = \mathcal{X}_{\tau(v)}$, while the joint distribution over all node features may exhibit arbitrary dependencies induced by the graph structure. Each edge $e = (u,v) \in \mathcal{E}$, with relation type $r = \phi(e)$, encodes a typed interaction between nodes $u$ and $v$. 
This general formulation recovers standard foundation model modalities as special cases (see \Cref{fig:gfms}). For LLMs, the graph reduces to a directed, homogeneous path graph with $|\mathcal{T}_V| = |\mathcal{T}_E| = 1$, where nodes correspond to token positions connected by directed edges forming a chain, and node features take values in a finite vocabulary mapped to continuous embeddings augmented with positional encodings. For ViTs, the graph is a fixed, homogeneous two-dimensional grid with $|\mathcal{T}_V| = |\mathcal{T}_E| = 1$, where each node corresponds to a patch carrying a feature vector $\mathbf{x}_v \in \mathbb{R}^{3 \times p \times p}$. For tabular foundation models, the graph degenerates to isolated nodes with no edges, where attention operates over feature columns rather than graph structure. In contrast, \method{} operates on general heterogeneous graphs with $|\mathcal{T}_V| \geq 1$ and $|\mathcal{T}_E| \geq 1$, without assuming any fixed topology.

In this work, we consider universes of arbitrary $\mathcal{G}$s.
We discuss in the appendix the trade-offs in designing $\mathcal{G}$, and in \Cref{sec:finetuning} how to extend it to accommodate new feature, node or edge types outside of it.
$\mathcal{G}$  may be a single connected graph or consist of multiple connected components, for example as a disjoint union of graphs. Since this is a property of the data representation rather than a conceptual distinction relevant to our framework, we treat 
$\mathcal{G}$  as a single graph, possibly with multiple connected components, and do not distinguish between these cases in the remainder of the work. 

We consider the node features $\mathbf{x}_v$ to be the token inputs to the model, regardless of whether they are obtained from raw data or through a learned tokenization, which is the standard for LLMs and recently suggested for graphs as well \citep{wang2025learninggraphquantizedtokenizers}. 
For the rest of this paper, we refer to tokens simply as nodes.
The initial representation of a node is its initial features  $\mathbf{h}^{(0)}_v\in \mathbb{R}^{d_{\tau(v)}}$ .
The representation of a node $v\in V$ in layer $\ell\in \{1,\dots, L\}$ is denoted by
$\mathbf{h}^{(\ell)}_v\in\mathbb{R}^{d_\ell}$, where $d_\ell$ is the hidden dimension at layer $\ell$.
We denote by $\mathcal{C}_v^a\subseteq V$ a \emph{context} of node $v$, defined as a set of nodes that are allowed to exchange information with $v$ according to some function $a:V\rightarrow 2^V$.
The context is not required to coincide with direct connectivity between nodes.
All attention mechanisms use scaled dot-product attention with head dimension $d_h = d_\ell / H$ where $H$ is the number of attention heads.
Bold lowercase letters denote vectors and bold uppercase letters denote matrices.

\section{The \method{} Transformer}
In this section, we introduce \method{} Transformer used in our framework.
\method{} Transformer is designed for billion-parameter GFMs, training on billion-scale heterogeneous graphs, with feasible resources. 
This design reflects three empirical properties of real-world heterogeneous graphs. First, node and edge type distributions may be highly imbalanced. Second, rare relation types may carry disproportionately strong signal, and third, nodes may have million-scale degrees.  

\begin{wrapfigure}{r}{0.3\textwidth}
    \centering
    \includegraphics[width=0.3\textwidth]{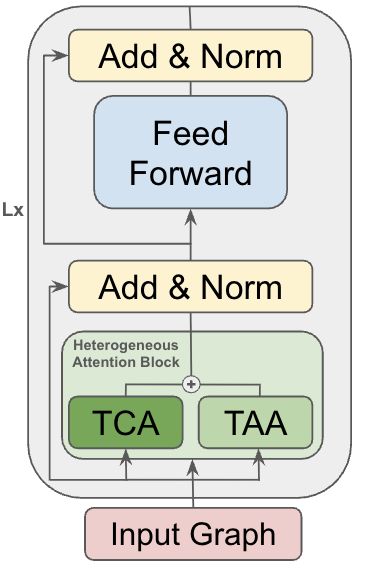}
    \caption{The \method{} Transformer block.}
    \label{fig:model}
\end{wrapfigure}
The \method{} Transformer updates node representations in each layer, based on their neighborhoods in the graph.
Each \method{} Transformer layer follows the standard transformer encoder
block structure \citep{vaswani2023attentionneed}, with residual connections, layer normalization, and a feed-forward network (FFNs), while replacing the all-pairs self-attention sub-block with two heterogeneous graph-aware masked attention modules. 

The \method{} introduces the \emph{Type-Conditioned Attention} (TCA) which also builds on type-specific attention transformations as previously suggested in Heterogeneous-Graph Transformer (HGT) \citep{hu2020heterogeneousgraphtransformer}, with a small yet crucial difference: a sparse softmax is applied to each type of neighbors separately, rather than to all neighbors at once. 
As softmax computation is often a bottleneck in self-attention due to the need to materialize and normalize the full attention matrix, leading to high memory bandwidth and I/O costs when applied to large neighborhoods.

To enable cross-type attention, \method{} Transformer also introduces a \emph{Type-Agnostic Attention} (TAA) component, which applies a shared attention among neighbors, with a sparse, fixed-degree neighbor sampling. This guarantees that efficiency is preserved, while also reducing the risk of overfitting to the graph structure in cases where the node degrees follow highly non-regular distributions~\citep{bechlerspeicher2024graphneuralnetworksuse}. Importantly, this component shares its attention matrices across all edge types, therefore introducing a relatively small number of additional parameters compared with the TCA component, yet it strictly increases the model expressivity as we prove in \Cref{thm:mm_exspressivness}. The final node representation in heterogeneous attention block is then a learned combination of the TAA and TCA. See \Cref{fig:model} for an overview of the \method{} Transformer block.

\textbf{Type-Conditioned Attention (TCA)} 

The TCA component performs masked self-attention operations within neighborhoods $\mathcal{C}^{\mathrm{tca}}_v$, limited to specific subsets of edge types. For a set $S \subseteq \mathcal{T}_E$, the node is allowed to attend only to nodes within $\mathcal{C}^{\mathrm{tca}}_v$ that are connected through edge types in $S$. Multiple sets $\mathcal{S}$ can be defined, and TCA applies an independent self-attention with respect to each set $S$, and then aggregates the resulting representations into a single one.

For every node $u \in \mathcal{C}^S_v$, we additionally incorporate edge features $\mathbf{x}_{uv}$ via an edge encoder
\( \mathbf{g}_{uv} = E(\mathbf{x}_{uv}) \in \mathbb{R}^{d_e} \).
We define:
\[
\mathbf{q}^{(S,\ell)}_v = \mathbf{W}^{(S,\ell)}_{Q}\,\mathbf{h}^{(\ell)}_v, \quad 
\mathbf{k}^{(S,\ell)}_u = \mathbf{W}^{(S,\ell)}_{K}\,\mathbf{h}^{(\ell)}_u, \quad 
\mathbf{v}^{(S,\ell)}_u = \mathbf{W}^{(S,\ell)}_{V}\,\mathbf{h}^{(\ell)}_u
\]

A learnable edge-attention vector $\mathbf{a}^{(S,\ell)} \in \mathbb{R}^{d_e}$ projects the edge features into a scalar bias for the attention logit. The attention weights for set $S$ are:
\[
\alpha^{(S,\ell)}_{uv} =
\frac{
\exp\!\left(
\frac{(\mathbf{q}^{(S,\ell)}_v)^\top \mathbf{k}^{(S,\ell)}_u}{\sqrt{d_h}}
+ (\mathbf{a}^{(S,\ell)})^\top \mathbf{g}_{uv}
\right)
}{
\sum_{u' \in \mathcal{C}^S_v}
\exp\!\left(
\frac{(\mathbf{q}^{(S,\ell)}_v)^\top \mathbf{k}^{(S,\ell)}_{u'}}{\sqrt{d_h}}
+ (\mathbf{a}^{(S,\ell)})^\top \mathbf{g}_{u'v}
\right)
}
\]

The set-specific representation of node $v$ for set $S$ is:
\[
\mathbf{h}^{(\ell,\mathrm{S})}_v =
\sum_{u \in \mathcal{C}^S_v}
\alpha^{(S,\ell)}_{uv}
\left(
\mathbf{v}^{(S,\ell)}_u
+
\mathbf{W}^{(S,\ell)}_{E}\,\mathbf{g}_{uv}
\right)
\]

The overall TCA representation of node $v$ in layer $\ell$ is obtained by aggregating its set-specific representations across all defined sets in $\mathcal{S}$:
\(
\mathbf{h}^{(\ell,\mathrm{tca})}_v =
\sum_{S \in \mathcal{S}, ~ \mathcal{C}^S_v \neq \emptyset} \mathbf{h}^{(\ell,\mathrm{S})}_v
\)

\textbf{Type-Agnostic Attention (TAA)}
The TAA component performs self-attention between a node and all other nodes in its TAA neighborhood $\mathcal{C}^{\mathrm{taa}}_v$. It first projects all nodes to the same embedding dimension.
This component is parameter-efficient as the attention matrices are shared across all node types, yet it can be expensive when applied to large neighborhoods. Therefore, we further apply a fixed sampling function $q$ over the neighborhood $\mathcal{C}^{\mathrm{taa}}_v$, $q: \mathcal{C}^{\mathrm{taa}}_v \rightarrow 2^{\mathcal{C}^{\mathrm{taa}}_v}$.
Formally, node representations are first mapped into a shared latent space.
At layer $\ell = 0$, each node type $\tau \in \mathcal{T}_V$ has a type-specific projection
$\mathbf{W}^{(0)}_{\tau} \in \mathbb{R}^{d_\ell \times d_{\tau}}$ that maps from the
type-specific feature dimension $d_\tau$ to the shared hidden dimension $d_\ell$.
For layers $\ell \geq 1$, all nodes already lie in $\mathbb{R}^{d_\ell}$ and the projection
becomes $\mathbf{W}^{(\ell)}_{\tau} \in \mathbb{R}^{d_\ell \times d_\ell}$.
In both cases, the projected representation is
\(
\widehat{\mathbf{h}}^{(\ell)}_v
=
\mathbf{W}^{(\ell)}_{\tau(v)}\,\mathbf{h}^{(\ell)}_v
\in\mathbb{R}^{d_\ell}.
\)
We then use \emph{shared} attention matrices across all nodes and edge types:
\(
\mathbf{W}^{(\ell)}_{Q},\mathbf{W}^{(\ell)}_{K},\mathbf{W}^{(\ell)}_{V}
\in\mathbb{R}^{d_h\times d_\ell}.
\)

For any neighbor $u\in\mathcal{C}_v$, we compute:
\(
\mathbf{q}^{(\ell)}_v = \mathbf{W}^{(\ell)}_{Q}\,\widehat{\mathbf{h}}^{(\ell)}_v,\quad
\mathbf{k}^{(\ell)}_u = \mathbf{W}^{(\ell)}_{K}\,\widehat{\mathbf{h}}^{(\ell)}_u,\quad
\mathbf{v}^{(\ell)}_u = \mathbf{W}^{(\ell)}_{V}\,\widehat{\mathbf{h}}^{(\ell)}_u
\)

We similarly incorporate edge features via $\mathbf{g}_{uv} = E(\mathbf{x}_{uv})$, with a shared edge-attention vector $\bar{\mathbf{a}}^{(\ell)} \in \mathbb{R}^{d_e}$. The type-agnostic attention weights are:
\begin{align*}
\beta^{(\ell)}_{uv} =
\frac{
\exp\!\left(
\frac{(\mathbf{q}^{(\ell)}_v)^\top \mathbf{k}^{(\ell)}_u}{\sqrt{d_h}}
+ \bar{\mathbf{a}}^{(\ell)\top} \mathbf{g}_{uv}
\right)
}{
\sum_{u'\in\mathcal{C}^{\mathrm{taa}}_v}
\exp\!\left(
\frac{(\mathbf{q}^{(\ell)}_v)^\top \mathbf{k}^{(\ell)}_{u'}}{\sqrt{d_h}}
+ \bar{\mathbf{a}}^{(\ell)\top} \mathbf{g}_{u'v}
\right)
}
\end{align*}

and the resulting TAA embedding is
\[
\mathbf{h}^{(\ell,\mathrm{taa})}_v =
\sum_{u\in\mathcal{C}^{\mathrm{taa}}_v}
\beta^{(\ell)}_{uv}
\left(
\mathbf{v}^{(\ell)}_u
+
\mathbf{W}^{(\ell)}_{E}\,\mathbf{g}_{uv}
\right)
\]
Finally, the TCA and TAA embeddings are combined into a
single output for node $v$, using an FFN  $\Phi^{(\ell)}$:
\(
\widetilde{\mathbf{h}}^{(\ell)}_v
=
\Phi^{(\ell)}\!\left(
\mathbf{h}^{(\ell,\mathrm{tca})}_v,\,
\mathbf{h}^{(\ell,\mathrm{taa})}_v
\right)
\in\mathbb{R}^{d_\ell},
\)

Following the heterogeneous attention sub-block, $\widetilde{\mathbf{h}}^{(\ell)}_v$ and $\mathbf{z}^{(\ell)}_v$ are passed through the remaining components of the transformer encoder block:

\begin{equation*}
\mathbf{z}^{(\ell)}_v=
\mathrm{LayerNorm}\!\left(
\mathbf{h}^{(\ell)}_v
+
\widetilde{\mathbf{h}}^{(\ell)}_v
\right), \quad
\mathbf{h}^{(\ell+1)}_v =
\mathrm{LayerNorm}\!\left(
\mathbf{z}^{(\ell)}_v
+
\mathrm{FFN}\!\left(\mathbf{z}^{(\ell)}_v\right)
\right).
\end{equation*}

The next theorem shows that a \method{} Transformer with both TCA and TAA is strictly more expressive than one lacking either component.
\begin{theorem} \label{thm:mm_exspressivness}
Consider a \method{} Transformer layer with hidden dimension $d_\ell=d$ and number of heads $H$, and heterogeneous attention sub-block TCA and TAA.
Let $\mathcal{F}_{\mathrm{GraphBFF}}, \mathcal{F}_{\mathrm{TAA}}, \mathcal{F}_{\mathrm{TCA}}$ be the sets of realizable functions by the \method{} Transformer with both TCA and TAA, with just TAA and with just TCA, respectively.
There exists a function $f$ such that:
$f \in \mathcal{F}_{\mathrm{GraphBFF}}$
but
$f \notin \mathcal{F}_{\mathrm{TAA}}$
and 
$f \notin \mathcal{F}_{\mathrm{TCA}}$.
\end{theorem}

We prove the theorem in the Appendix. To prove it, we construct a function that TCA fails to realize due to softmax normalization within edge-type subsets effectively masking the relative cardinality of neighbor sets, and TAA fails to realize due to shared parameters rendering it blind to specific edge-type distinctions, yet TAA and TCA together can realize it.
In \Cref{sec:evaluation} we show that using both TCA and TAA is also preferable empirically, leading to better generalization.

\subsection{Pre-Training}
\label{sec:pretrain}
Many self-supervised objectives have been proposed for graphs \citep{velivckovic2019dgi, hou2022graphmae2, you2020graphcl, hu2020strategies}, typically to inject stronger, task- or structure-specific inductive biases.
Inspired by recent results showing that scale can compensate for inductive bias~\citep{brehmer2025doesequivariancematterscale, tay2022scalinglawsvsmodel, Bahri_2024} we conjecture that a simple masked link prediction~\citep{kipf2016variational, hou2022graphmae} applied at billion-scale is sufficient for obtaining effective GFMs.

\textbf{Pre-training objective.} We use masked link prediction as the self-supervised objective. Given a graph $G = (\mathcal{V}, \mathcal{E})$, we randomly sample a set of positive edges $\mathcal{E}^+ \subset \mathcal{E}$ and a corresponding set of negative edges $\mathcal{E}^- = \{(u,v) : (u,v) \notin \mathcal{E}\}$ drawn uniformly at random with a 1:1 negative sampling ratio. The supervised edges are removed from the input graph during the forward pass. For each edge $(u,v)$, the model produces a score $s_{uv}$ by concatenating the source and target node embeddings and passing them through a two-layer MLP. The training loss is the binary cross-entropy:
\[
\mathcal{L} = -\frac{1}{|\mathcal{E}^+| + |\mathcal{E}^-|} \left( \sum_{(u,v) \in \mathcal{E}^+} \log \sigma(s_{uv}) + \sum_{(u,v) \in \mathcal{E}^-} \log(1 - \sigma(s_{uv})) \right),
\]
where $\sigma$ is the sigmoid function.

\textbf{Batching.} Batching graph data is commonly used to enable machines with constrained memory to load and process large-scale graphs. Since we address heterogeneous and potentially type-skewed graphs, we want to ensure that training is done over batches that represent the graph well in terms of node and edge type distributions. Furthermore, type-skewness might result in biased training due to highly common edge types and noisy gradients for rare types. To address these issues, we deploy a two-stage batching process.

\textbf{Storage-level batching.}
As transferring data from storage into memory is costly, we aim to maximize batch size while respecting a memory budget $M$. Standard practice is to partition the graph into clusters~\citep{Chiang_2019}, e.g., via Leiden~\citep{traag2019leiden}, however, such methods result in varying cluster size and node/edge-type distribution. This may result in non-representative clusters, and in case of large effective batch sizes in multi-machine pre-training, under-fill the batch capacity and create early distributional bias that destabilizes optimization. To address this, we propose \textit{KL-Batching},
which operates as follows: (1) use a graph clustering algorithm to partition the graph into small, disjoint clusters, (2) compute, for each cluster, an empirical distribution $p_k$ over a chosen categorical attribute (e.g., node or edge type) and measure its representativeness via $\mathrm{KL}(p_k\|p_G)$ against the global reference distribution $p_G$ (optionally combining multiple attributes with weighted KL terms), (3) construct batches by greedily aggregating whole clusters into batches, by increasing KL value, until the estimated batch cost reaches the memory limit $M$. This yields memory-efficient batches that better match global type distributions and utilize memory more effectively. Clusters are traversed in ascending KL order, so the first batches best match the global distribution. Since pre-training typically uses only a subset of the full graph, only the head batches with the most representative distributions are used, and tail batches with high-KL anomalous clusters are discarded. The full mathematical formulation is provided in the Appendix.

\textbf{GPU-level batching.} KL-batching yields batches that can be loaded into memory, but must be further subdivided into mini-batches that fit GPU memory. Randomly splitting the KL-batches for highly skewed graphs may produce mini-batches overwhelmingly composed of common edge types, yielding a biased training signal and noisy gradients for rare relations.
To avoid this, we propose \emph{Round-Robin Batching (RRB)}. In RRB, we group supervision edges by their type, and iterate over types in a fixed cyclic order. At each step, we form a mini-batch by sampling supervision edges (and their negative counterparts) of the current type, and materialize their neighborhoods from the KL-batch, excluding the supervised edges themselves. When a minority edge type is exhausted within an epoch, its edges are reiterated from the beginning, ensuring all types receive continued supervision throughout training. This simple scheduling leads to more stable training and better coverage of the heterogeneous edge-type space.

\subsection{Fine-Tuning and Extending $\mathcal{G}$}
\label{sec:finetuning}
Fine-tuning can be performed in both supervised and unsupervised settings,
and serves several distinct use cases. The most straightforward scenarios
involve adapting a pretrained \method{} to new data drawn from
the same underlying universe, or multi-task supervised training to improve performance on one
or more downstream tasks. In all of these cases, standard parameter‑efficient
fine-tuning methods such as LoRA~\citep{hu2021lora} can be applied to
\method, with the additional flexibility of constraining updates
to specific type-indexed matrices for node and edge types, and to selected
sub-modules such as the TAA or TCA components.

Downstream tasks may introduce features, node types, or edge
types that were not included in $G$ as they did not exist during pre-training, or as a modeling choice, for example if these types very sparse over the graph, or relevant only to a small amount of downstream tasks.
An approach to avoid the need to extend $\mathcal{G}$, is to map new features or node types into existing ones. These approaches were extensively researched in the TabFMs domain, with recent works applying them to graphs as well. We discuss these approaches in \Cref{sec:related}.
In an extension of $\mathcal{G}$, we extend it to a richer universe $G'$ with new features or types, by introducing new learned weights to the \method{} Transformer. We then train these weights through fine-tuning, possibly with the rest of the model frozen \citep{tai2020exbert}.  This approach may be preferable when some feature or node types are relevant for only a small fraction of downstream tasks, or when they occur in only a small fraction of the data.

\section{Evaluation} \label{sec:evaluation}
\begin{table*}[t]
\caption{A comparison between the GFM with up to 2 layers NN prediction head (Probing) and task-specific heterogeneous transformers (GraphBFF, HGT, HAN, GraphGPS). For each model, we compare 3 context sizes corresponding to ego-graphs of varying radius centered on the target node or edge. The best performance overall is highlighted in bold. The best performance for each context size for the task-specific baselines is highlighted in blue, green and red.}
\label{tab:probing}
\centering

\fontsize{6}{8}\selectfont
\setlength{\tabcolsep}{2pt}
\begin{tabular}{p{0.1cm} *{11}{c}}
\toprule
\multicolumn{1}{c}{\textbf{}} & \textbf{Model} &
\textbf{Task 1} & \textbf{Task 2} & \textbf{Task 3} & \textbf{Task 4} &
\textbf{Task 5} & \textbf{Task 6} & \textbf{Task 7} & \textbf{Task 8} &
\textbf{Task 9} & \textbf{Task 10} \\
& -\textbf{Context}&
\tiny PRAUC & \tiny PRAUC & \tiny PRAUC & \tiny PRAUC &
\tiny PRAUC & \tiny PRAUC & \tiny MAE & \tiny PRAUC &
\tiny PRAUC   & \tiny PRAUC   \\
\midrule
\multirow{10}{*}{\centering\arraybackslash\rotatebox{90}{Task-Specific}}
& NN   & 65.61$_{\pm 1.01}$ &  72.06$_{\pm0.01}$&    51.43$_{\pm0.01}$&   63.05$_{\pm3.80}$& 91.20$_{\pm4.00}$& 58.51$_{\pm0.35}$ & 0.032$_{\pm0.01}$& 85.76$_{\pm1.21}$ & 63.66$_{\pm3.26}$& 55.48$_{\pm0.79}$ \\
\cmidrule(l){2-12}
& HGT-1  & 61.23$_{\pm1.95}$ & 66.85$_{\pm2.27}$ & 71.83$_{\pm0.20}$ & 90.43$_{\pm0.18}$ &60.99$_{\pm12.3}$  & 71.06$_{\pm2.02}$  &  0.236$_{\pm0.02}$& 69.54$_{\pm1.21}$ & 30.08$_{\pm5.99}$ & 60.78$_{\pm0.49}$ \\
& HAN-1  & 61.73$_{\pm1.85}$ & 62.83$_{\pm1.50}$ & 64.12$_{\pm0.88}$ & 94.70$_{\pm0.30}$ & \textcolor{blue}{76.34$_{\pm5.81}$}& 70.10$_{\pm0.94}$ &0.083$_{\pm0.01}$& 67.64$_{\pm2.35}$ & 36.87$_{\pm2.31}$ & 60.59$_{\pm3.28}$ \\
& GraphGPS-1  & 62.08$_{\pm1.43}$ & 67.42$_{\pm1.15}$ & 72.51$_{\pm0.63}$ & 94.28$_{\pm0.42}$ & 74.82$_{\pm6.23}$ & 71.73$_{\pm1.48}$ & 0.089$_{\pm0.01}$ 
& 70.12$_{\pm1.85}$ & 35.42$_{\pm3.17}$ & 61.24$_{\pm1.05}$ \\ 
& \method{}-1  & \textcolor{blue}{64.77$_{\pm1.70}$} & \textcolor{blue}{67.94$_{\pm0.24}$}  & \textcolor{blue}{73.86$_{\pm0.71}$} & \textcolor{blue}{95.34$_{\pm0.51}$} & \textcolor{blue}{75.70$_{\pm4.12}$}& \textcolor{blue}{75.47$_{\pm0.69}$}& \textcolor{blue}{0.073$_{\pm0.01}$} &  \textcolor{blue}{72.47$_{\pm4.75}$}&  \textcolor{blue}{42.62$_{\pm4.80}$} &  \textcolor{blue}{62.92$_{\pm3.22}$}\\
\cmidrule(l){2-12}
& HGT-2  &  64.73$_{\pm0.61}$& 72.34$_{\pm0.86}$ & 71.94$_{\pm3.31}$ &  \textcolor{green}{95.02$_{\pm1.41}$}&\textcolor{green}{63.26$_{\pm18.8}$}  &  \textcolor{green}{72.86$_{\pm2.18}$} &0.186$_{\pm0.02}$& \textcolor{green}{70.67$_{\pm0.87}$} & 30.71$_{\pm7.18}$ &  58.63$_{\pm 3.55}$ \\
& HAN-2  & 61.65$_{\pm1.88}$ & 76.26$_{\pm0.79}$& 67.14$_{\pm0.96}$ & 94.70$_{\pm5.87}$ &43.26$_{\pm5.97}$& 69.74$_{\pm0.95}$ &  0.161$_{\pm0.01}$& 59.38$_{\pm2.04}$ & 31.92$_{\pm13.9}$ & \textcolor{green}{63.23$_{\pm2.03}$}  \\
& GraphGPS-2  & 65.38$_{\pm0.94}$ & 75.61$_{\pm1.22}$ & 72.85$_{\pm2.46}$ & 94.15$_{\pm2.03}$ & 57.42$_{\pm8.35}$ & 70.48$_{\pm1.65}$ & 0.174$_{\pm0.02}$ & 66.23$_{\pm2.41}$ & 33.24$_{\pm5.47}$ & 61.85$_{\pm2.74}$ \\
& \method{}-2  &  \textcolor{green}{69.81$_{\pm1.51}$}&  \textcolor{green}{81.48$_{\pm0.66}$}& \textcolor{green}{79.45$_{\pm1.82}$}  & \textcolor{green}{94.41$_{\pm2.61}$} & \textcolor{green}{63.15$_{\pm3.97}$}& \textcolor{green}{72.29$_{\pm4.21}$} & \textcolor{green}{0.156$_{\pm0.02}$}  & \textcolor{green}{71.13$_{\pm3.78}$} & \textcolor{green}{37.41$_{\pm4.75}$} &  63.76$_{\pm3.03}$\\
\cmidrule(l){2-12}
& HGT-3  & \textcolor{red}{76.88$_{\pm 2.14}$} & \textcolor{red}{77.66$_{\pm0.16}$} & 67.67$_{\pm0.95}$ & 91.75$_{\pm0.15}$ &45.59$_{\pm4.55}$  & \textcolor{red}{69.51$_{\pm2.56}$} & \textcolor{red}{0.135$_{\pm0.01}$} & 71.18$_{\pm0.78}$ & 33.88$_{\pm6.76}$ & 57.85$_{\pm3.83}$ \\
& HAN-3  & 74.32$_{\pm0.96}$ &74.75$_{\pm3.21}$  & 72.47$_{\pm1.18}$ & \textcolor{red}{94.83$_{\pm0.69}$}  &  \textcolor{red}{53.98$_{\pm17.5}$} & \textcolor{red}{69.46$_{\pm1.74}$} & 0.138$_{\pm0.02}$ &70.23$_{\pm1.08}$  & 30.71$_{\pm9.95}$ & 57.58$_{\pm2.21}$ \\
& GraphGPS-3  & 75.64$_{\pm1.67}$ & 76.93$_{\pm1.44}$ & 71.28$_{\pm1.35}$ & 92.18$_{\pm0.87}$ & 49.62$_{\pm8.21}$ & 69.28$_{\pm1.82}$ & 0.142$_{\pm0.02}$ & 71.42$_{\pm1.23}$ & 32.45$_{\pm5.82}$ & 58.92$_{\pm2.17}$ \\
& \method{}-3  & 74.42$_{\pm 1.74}$ & \textcolor{red}{79.34$_{\pm2.17}$} & \textcolor{red}{80.15$_{\pm1.75}$} & 93.56$_{\pm0.49}$ & \textcolor{red}{54.16$_{\pm0.01}$}& \textcolor{red}{69.65$_{\pm1.19}$} & 0.146$_{\pm0.03}$ & \textcolor{red}{78.83$_{\pm2.80}$} & \textcolor{red}{42.54$_{\pm3.40}$}&  \textcolor{red}{63.24$_{\pm2.39}$}\\

\cmidrule(l){1-12}
\multirow{3}{*}{\centering\arraybackslash\rotatebox{90}{Probing}}
& GFM-1 & 84.07$_{\pm0.37}$ &  82.58$_{\pm0.01}$&  83.03$_{\pm0.07}$&  98.31$_{\pm0.21}$&  60.70$_{\pm0.39}$  &75.85$_{\pm1.90}$  & 0.021$_{\pm0.01}$& 91.48$_{\pm0.92}$ &\textbf{95.19$_{\pm0.20}$}  & \textbf{79.65$_{\pm1.46}$}\\
& GFM-2 & 83.50$_{\pm0.13}$ &  \textbf{85.73$_{\pm0.05}$}&  84.36$_{\pm0.04}$&  98.36$_{\pm0.25}$&  83.93$_{\pm0.69}$&  \textbf{82.05$_{\pm0.04}$}&\textbf{0.018$_{\pm0.01}$} & 92.78$_{\pm0.20}$& 89.03$_{\pm0.58}$ & 60.77$_{\pm4.64}$ \\
& GFM-3 & \textbf{88.13$_{\pm0.15}$}& 83.00$_{\pm0.01}$&  \textbf{86.14$_{\pm0.11}$}&  \textbf{99.48$_{\pm0.19}$}&  \textbf{95.60$_{\pm0.32}$}&  75.54$_{\pm1.04}$& 0.021$_{\pm0.01}$ &  \textbf{95.39$_{\pm0.23}$}& 85.70$_{\pm0.56}$ & 78.48$_{\pm0.01}$ \\
\bottomrule
\end{tabular}

\end{table*}

In this section, we present an extensive evaluation of a 1.4 billion-parameters \method{} Transformer pre-trained on one billion samples. We evaluate it over diverse downstream tasks, spanning node and edge level classification and regression. We consider probing, few-shot probing, and zero-shot settings. We also perform an ablation study, evaluating a \method{} Transformer with only the TCA or the TAA components.

\textbf{Data and Tasks} We use a real-world billion-scale \emph{Enterprise} graph with ${\sim}50$ billion nodes and edges, with $12$ node types and $20$ relation types, including financial, social, business, infra, and more. Each node type is associated with its own features.
We use 10 diverse real-world industrial tasks including node and edge classification and regression, denoted as Task$1$ to Task$10$, where tasks $1{-}7$ are node-level tasks, and tasks $8{-}10$ are edge-level tasks. The tasks span finance, security, social and Attribution tasks, and defined over different node and edge types. The Enterprise graph contains roughly 50 billion nodes and edges, of which one billion are used as supervision edges for pre-training.
All the tasks' graphs are not part of the pre-training data, nor their nodes are connected to nodes in the pre-training data. All tasks are out-of-distribution with respect to the pre-training data. The amount of labels for the full data setting ranges from a minimum of $98$ to a maximum of $50k$ labels.
More details on the datasets including statistics are provided in the Appendix.

\textbf{Setting} We use the \method{} recipe to pre-train a 1.4 billion-parameter \method{} Transformer.
As the neighborhoods for the \method{} Transformer, we use direct neighbors for $\mathcal{N}^{\mathrm{tca}}$, i.e., nodes connected by a single edge. For $\mathcal{N}^{\mathrm{taa}}$, we use nodes within two hops of the target node. Within the TAA component, sampling function $q$ selects up to 10 random neighbors per hop. This accommodates graph scale and mitigates structural overfitting, a common issue for sparse graph transformers operating on large-variance node-degree distributions~\citep{bechlerspeicher2024graphneuralnetworksuse}. The TCA component is the primary driver of model capacity, comprising about 85\% of the total parameters.
We use node-type distribution for the KL-Batching.
Pre-training used eight Nvidia B200 machines and was completed in at most 12 hours with $64$ GPUs.
As evaluation metrics, we use PRAUC for classification and MAE for regression.
Across the evaluation, we used as context for the GFM the ego-graph~\citep{wasserman1994social} centered around target nodes or edges. We evaluate ego-graphs denoted as \textit{$k$-hops} with radius $k{\in}\{1,2,3\}$ to also examine how the context size affects performance.
Each task has its own train (80\%) validation (10\%) and test (10\%) splits.
We perform a grid search with $52$ hyper-param configurations for each setting, select the best configuration over the validation set, and report the mean and standard-deviation of the measured metric over the test set with 3 random seeds using the selected configuration. More hyper-parameters details are provided in the Appendix.

\begin{figure*}[t]
  \centering
  \includegraphics[width=\textwidth]{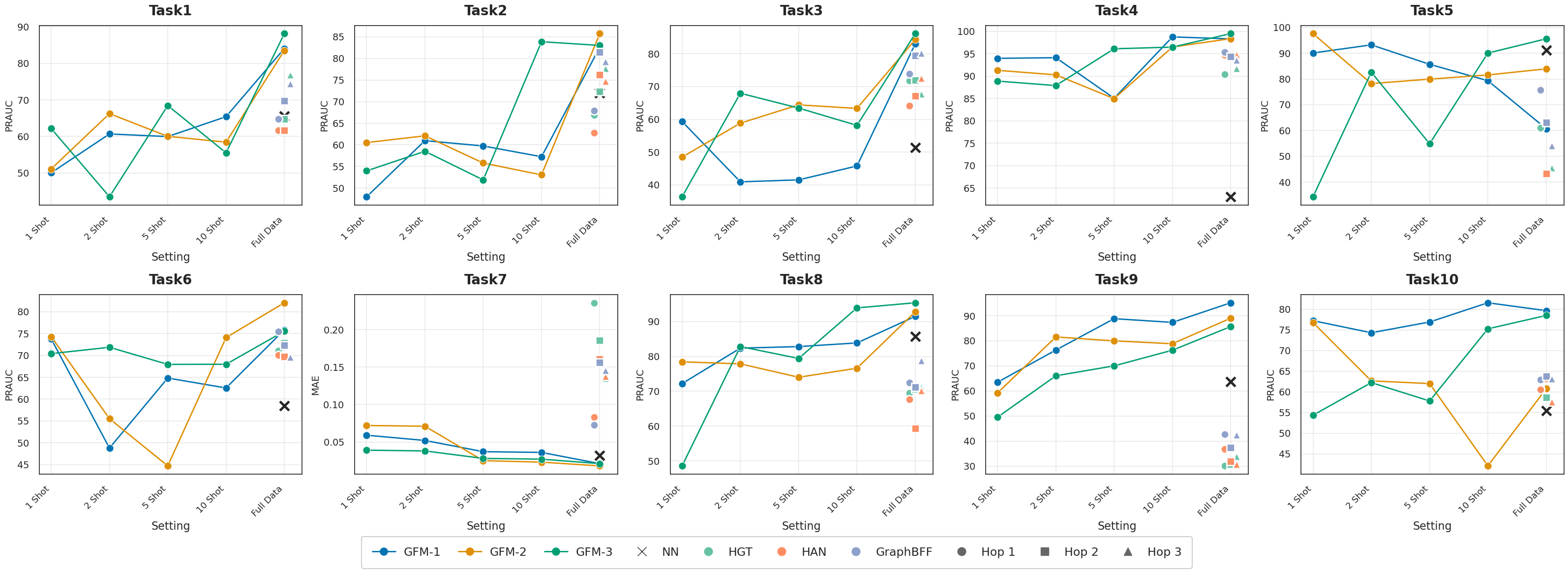}
  \caption{Performance across Tasks 1--10 under different few-shot settings. Lines show GFM-1/2/3. Markers at the full-data point indicate task-specific baselines with context size encoded by marker shape.}
  \label{fig:few_shot}
\end{figure*}

\textbf{Probing} We evaluate all ten tasks using (i) one-layer linear probes and (ii) up to two-layer non-linear probes. The GFM weights remain frozen and only the probe is trained. In particular, we train a vanilla feed-forward NN probe on top of the frozen GFM and compare it against task-specific graph models of comparable size to the NN probe, all trained on the same downstream task data.
We compare against a feed-forward neural network (NN) probe, HGT~\citep{hu2020heterogeneousgraphtransformer}, HAN~\citep{wang2021heterogeneousgraphattentionnetwork}, GraphGPS~\citep{rampasek2022gps}, and \method{} Transformer. Note that these graph-transformer baselines are trained on the input graph information, namely, the task-specific graph structure. In contrast, in the GFM probing setup these are only forwarded through the frozen GFM, and the NN probe is trained solely on the resulting GFM output embeddings.
For few-shot, we examine the performance of the NN probe over the GFM with 1, 2, 5, and 10 training samples per class, randomly sampled from the training set. For regression, we use 1, 2, 5, and 10 samples.

\subsection{Results}
\Cref{tab:probing,fig:few_shot} show that a simple NN probe on frozen \method{} representations outperforms all task-specific models in 10 out of 10 tasks, with gains of up to 31 PRAUC points. For instance, on Task 1, the best task-specific model HGT-3 reaches 76.88 while GFM Probing achieves 88.13. On Task 9, the gap exceeds 31 points with NN at 63.66 versus 95.19 for GFM Probing. Optimal context size is task-dependent, and \method{} consistently outperforms HGT, HAN, and GraphGPS as a task-specific baseline. A per-layer time and memory complexity comparison with HAN and HGT is provided in \Cref{app:time_complexity}, and a representational analysis examining Dirichlet Energy dynamics and in-context structural interrogation across layers is provided in \Cref{app:representational_analysis}.

\begin{table*}[tb]
\caption{Ablation study evaluating full-data probing with only TCA or TAA. $\Delta_{best}$ is the gap between the best ablation context and the full GFM from \Cref{tab:probing}.}
\label{tab:gfm_ablation}
\centering
\scriptsize
\setlength{\tabcolsep}{3pt}
\scriptsize
\setlength{\tabcolsep}{2pt}
\begin{tabular}{l *{10}{c}}
\toprule
\textbf{Model} & \textbf{T1} & \textbf{T2} & \textbf{T3} & \textbf{T4} & \textbf{T5} & \textbf{T6} & \textbf{T7} & \textbf{T8} & \textbf{T9} & \textbf{T10} \\
& \tiny PRAUC & \tiny PRAUC & \tiny PRAUC & \tiny PRAUC & \tiny PRAUC & \tiny PRAUC & \tiny MAE & \tiny PRAUC & \tiny PRAUC & \tiny PRAUC \\
\midrule
TCA-1 & 80.92$_{\pm.48}$ & 79.86$_{\pm.22}$ & 80.41$_{\pm.31}$ & 97.62$_{\pm.28}$ & 56.44$_{\pm.55}$ & 71.03$_{\pm1.6}$ & .027$_{\pm.01}$ & 88.73$_{\pm.61}$ & 90.84$_{\pm.41}$ & 73.28$_{\pm1.8}$ \\
TCA-2 & 79.71$_{\pm.29}$ & 82.63$_{\pm.18}$ & 81.92$_{\pm.20}$ & 97.71$_{\pm.33}$ & 79.54$_{\pm.74}$ & 77.14$_{\pm.34}$ & .024$_{\pm.01}$ & 89.86$_{\pm.33}$ & 84.62$_{\pm.77}$ & 55.33$_{\pm3.9}$ \\
TCA-3 & 84.65$_{\pm.25}$ & 80.34$_{\pm.12}$ & 83.02$_{\pm.19}$ & 98.84$_{\pm.23}$ & 91.02$_{\pm.46}$ & 70.92$_{\pm1.2}$ & .026$_{\pm.01}$ & 92.18$_{\pm.28}$ & 80.61$_{\pm.68}$ & 72.05$_{\pm.82}$ \\
\midrule
$\Delta_{best}$ & -3.48{\color{red}$\downarrow$} & -3.10{\color{red}$\downarrow$} & -3.12{\color{red}$\downarrow$} & -0.64{\color{red}$\downarrow$} & -4.58{\color{red}$\downarrow$} & -4.91{\color{red}$\downarrow$} & +.006{\color{red}$\uparrow$} & -3.21{\color{red}$\downarrow$} & -4.35{\color{red}$\downarrow$} & -6.37{\color{red}$\downarrow$} \\
\cmidrule(lr){1-11}
TAA-1 & 61.44$_{\pm.72}$ & 58.92$_{\pm.41}$ & 60.23$_{\pm.55}$ & 90.35$_{\pm.66}$ & 41.08$_{\pm.83}$ & 52.77$_{\pm2.1}$ & .044$_{\pm.02}$ & 72.61$_{\pm1.1}$ & 77.43$_{\pm.95}$ & 49.92$_{\pm2.4}$ \\
TAA-2 & 59.80$_{\pm.54}$ & 62.15$_{\pm.49}$ & 61.07$_{\pm.43}$ & 90.61$_{\pm.58}$ & 58.02$_{\pm1.1}$ & 56.48$_{\pm.88}$ & .041$_{\pm.02}$ & 74.38$_{\pm.84}$ & 70.05$_{\pm1.2}$ & 38.40$_{\pm5.0}$ \\
TAA-3 & 64.92$_{\pm.49}$ & 59.61$_{\pm.37}$ & 62.84$_{\pm.46}$ & 92.18$_{\pm.44}$ & 71.35$_{\pm.92}$ & 52.10$_{\pm1.7}$ & .043$_{\pm.02}$ & 78.92$_{\pm.63}$ & 66.21$_{\pm1.3}$ & 47.38$_{\pm1.7}$ \\
\midrule
$\Delta_{best}$ & -23.2{\color{red}$\downarrow$} & -23.6{\color{red}$\downarrow$} & -23.3{\color{red}$\downarrow$} & -7.30{\color{red}$\downarrow$} & -24.3{\color{red}$\downarrow$} & -25.6{\color{red}$\downarrow$} & +.023{\color{red}$\uparrow$} & -16.5{\color{red}$\downarrow$} & -17.8{\color{red}$\downarrow$} & -29.7{\color{red}$\downarrow$} \\
\bottomrule
\end{tabular}
\end{table*}

The necessity of the full \method{} Transformer architecture is supported by \Cref{tab:gfm_ablation}. Removing the TAA component (TCA-only) causes a consistent drop across all 10 tasks, with the best TCA-only configuration falling 3--6 PRAUC points below the full model on most tasks. Removing the TCA component (TAA-only) leads to much larger degradation, with drops of 17--30 PRAUC points, confirming that TCA is the primary driver of model capacity. Notably, even the TCA-only variant outperforms all task-specific baselines on every task, indicating that the pretrained TCA representations alone are highly effective.

\section{Neural Scaling Laws}\label{sec:scaling_laws}

In this section, we present neural scaling laws for GraphBFF.
Prior works on LLMs have shown that test loss follows predictable power-law relationships in both model size and data size, with clear transitions between data-limited and model-limited regimes as one of these resources is held fixed and the other is scaled \citep{kaplan2020scaling}. Here we show that GFMs follow the same trends, and show that gains are smooth when scaling model and data together, but saturate when one is fixed. 

\textbf{Setup.}
We begin with necessary notation, aligned with the notation in \cite{kaplan2020scaling}. Let $N$ denote the number of model parameters and let $D$ denote the number of
distinct supervised training edges, see the Appendix for discussion on the definition of $D$.
For each pair $(N, D)$, we define $L(N, D)$ to be
the best validation loss achieved when training on exactly $D$ samples with a model of size $N$.
We fit the joint scaling-law form from \cite{kaplan2020scaling}:

\begin{equation}\label{eq:scaling}
L(N,D) \;=\; L_{\infty}
\;+\;\left(\frac{N_c}{N}\right)^{\alpha_N}
\;+\;\left(\frac{D_c}{D}\right)^{\alpha_D},
\end{equation}

where $L_{\infty}$ is the irreducible loss floor. Here, $N_c$ and $D_c$ are normalization constants that characterize the model and data scales at which the corresponding terms begin to dominate the loss, and are estimated by fitting the scaling law to empirical performance curves.
Overall, we fit five parameters: the irreducible loss floor \(L_\infty\), the normalization constants \(N_c\) and \(D_c\), and the scaling exponents \(\alpha_N\) and \(\alpha_D\). Therefore, we demonstrate that with only $5$ parameters, the scaling law effectively models the validation loss across multiple orders of magnitude in both model size and dataset size.

We fit scaling behavior using logarithmically spaced data sizes with $\Delta \log_{10} D \approx 0.5$, and model sizes with
$\Delta \log_{10} N \approx 0.45$, providing sufficient resolution
to reliably estimate scaling exponents across regimes.
We consider nine model sizes spanning four orders of magnitude from $10^6$ to $4 \times 10^9$ parameters, and eight data sizes spanning five orders of magnitude from $3 \times 10^5$ to $10^9$ supervision edges.
The data we use is the enterprise data from \Cref{sec:evaluation}.
For each $(N,D)$ configuration, we train a \method{} Transformer using the complete \method{} recipe. We use four learning rates optimize with a sufficient number of epochs as done in \cite{kaplan2020scaling}. 
We evaluate the loss on a fixed holdout set of approximately $6$ million edges, drawn from a connected component that is not part of the training data.
Each training step processes a batch of $1024$ supervised edges.
For a given data size $D$, we define one epoch as a single pass over the $D$ distinct training edges, corresponding to $\lceil D / 1024 \rceil$ optimization steps.
Learning-rate schedules are parameterized by the total number of training steps for each run, with a warmup phase
($1$--$3\%$ of steps) as in~\cite{hoffmann2022training}.
The model and neighborhoods definitions are the same as in \Cref{sec:evaluation}.
\textbf{Results.}
\Cref{fig:scaling_laws_data} shows that the loss $L(N, D)$ varies predictably with the dataset size D and model
size N according to \Cref{eq:scaling}. 
This figure reveals both data- and model-bottlenecks: beyond a certain point, increasing \(D\) at fixed \(N\) yields negligible improvements in loss, and likewise increasing \(N\) at fixed \(D\) eventually provides little benefit. These results show that continued loss reductions require scaling data and model capacity together.
The test loss $L(\cdot)$ of a \method{} Transformer can be predicted using a power-law when performance is limited by $N$ or $D$. 
For models with a limited number of parameters, trained to convergence on sufficiently large datasets:
\(
L(N) = \left(\frac{N_c}{N}\right)^{\alpha_N},
\alpha_N \approx 0.703,
N_c \approx 2.1 \times 10^{4}.
\)
For large enough models trained with a limited dataset size:
\(
L(D) = \left(\frac{D_c}{D}\right)^{\alpha_D},
\alpha_D \approx 0.188,
D_c \approx 4.7.
\)
We also observe that larger GFMs attain a given test loss using fewer training examples than smaller models.

\begin{figure*}[t!]
    \centering
    \begin{subfigure}{0.3\textwidth}
        \centering
        \includegraphics[width=\textwidth]{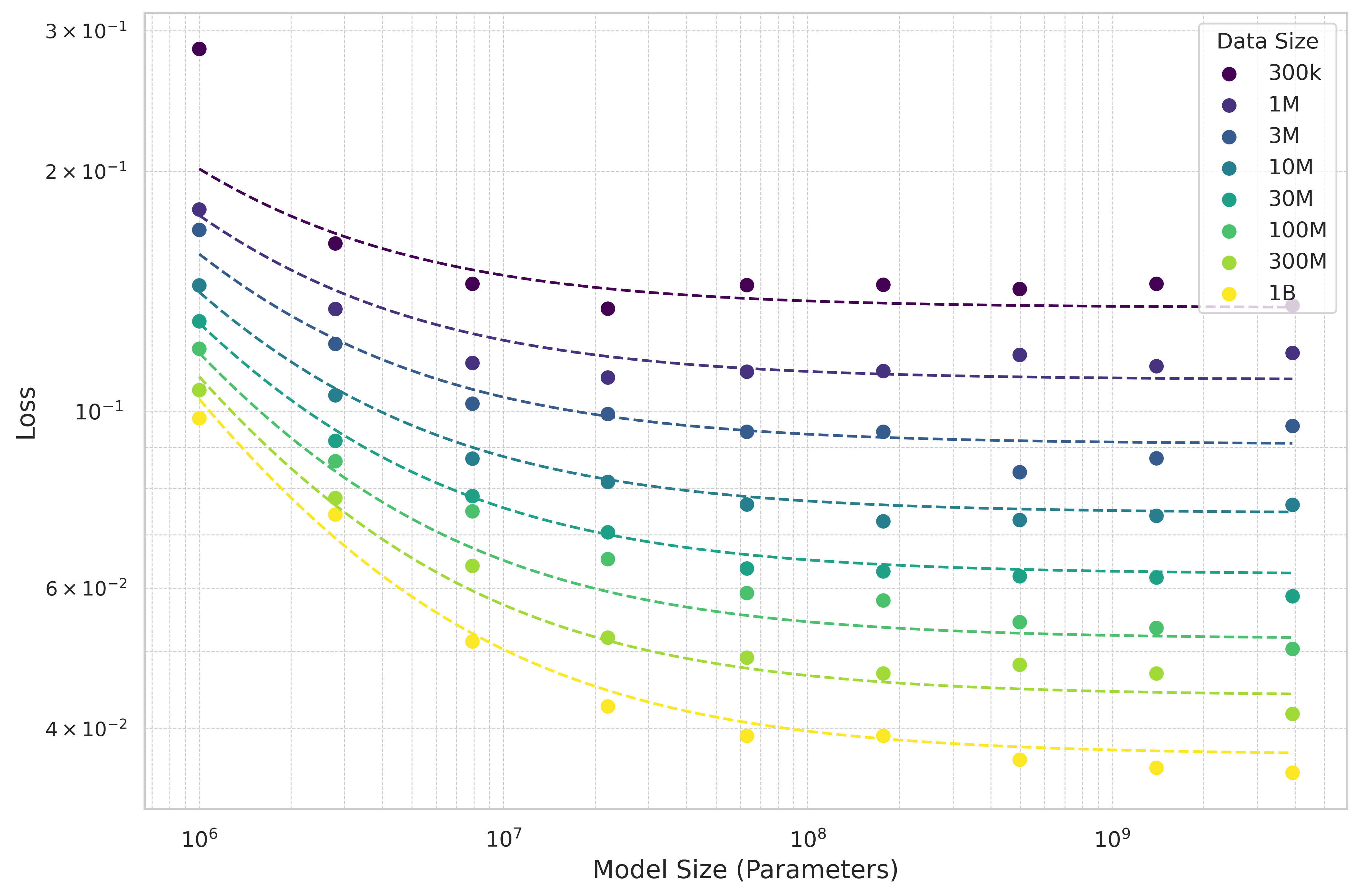}
        \caption{Validation loss as a function of the model size, for a fixed data size.}
    \end{subfigure}\hfill
    \begin{subfigure}{0.3\textwidth}
        \centering
        \includegraphics[width=\textwidth]{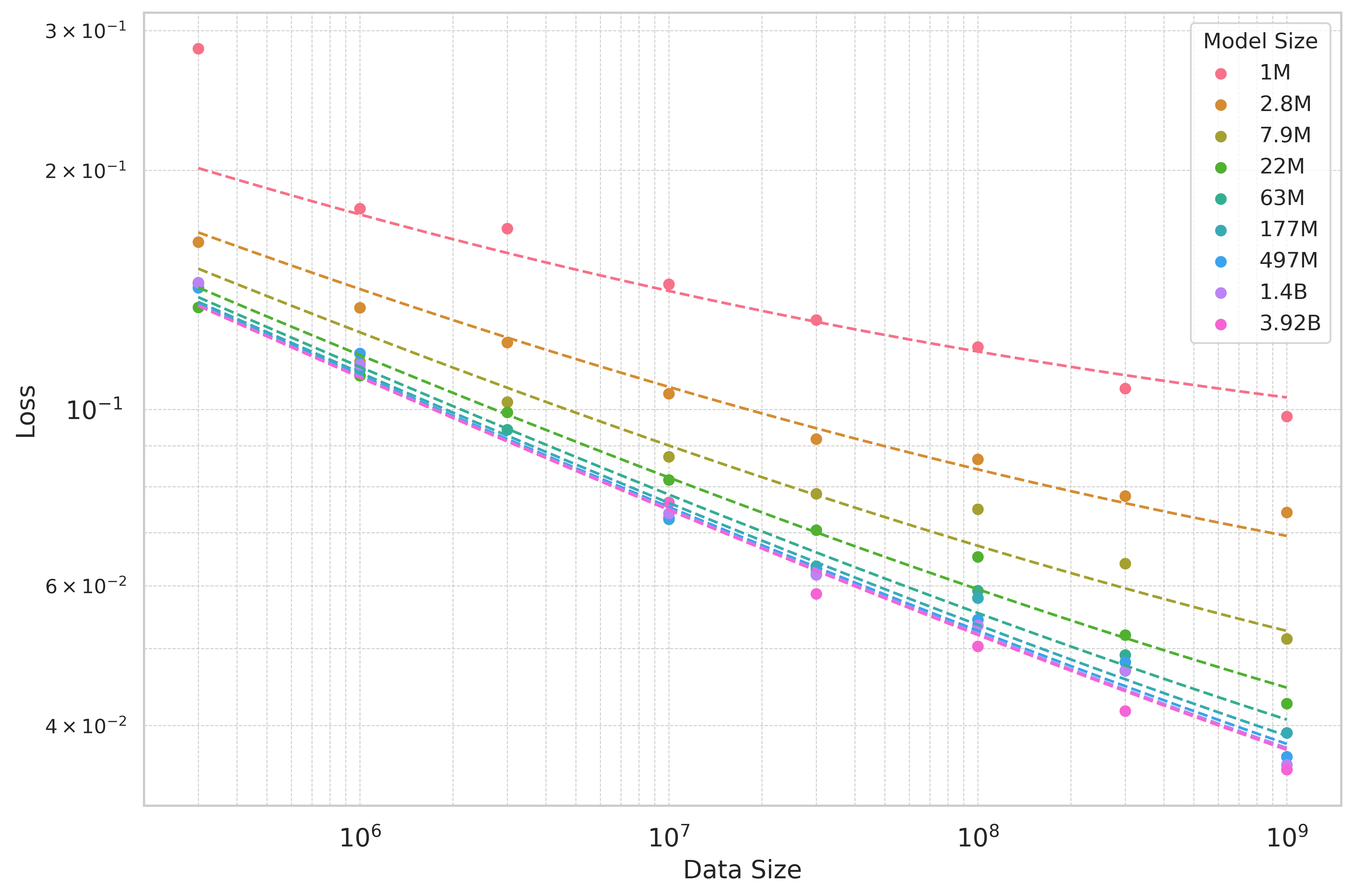}
        \caption{Validation loss as a function of the data size, for a fixed data model.}
    \end{subfigure}\hfill
    \begin{subfigure}{0.3\textwidth}
        \centering
        \includegraphics[width=\textwidth]{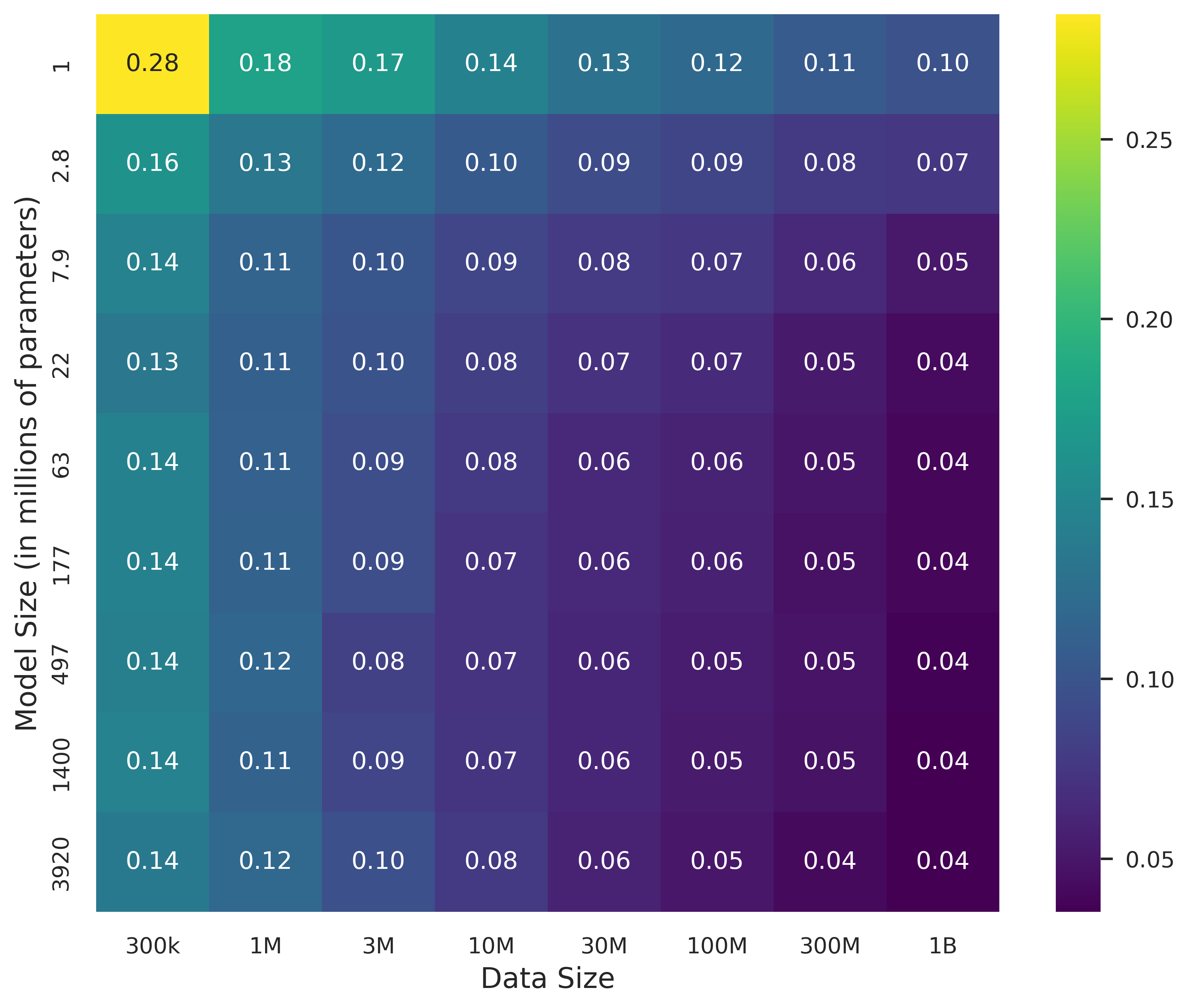}
        \caption{Validation loss as a function of the model size and data size.}
    \end{subfigure}

    \caption{The early-stop validation loss as a function of model size and data size. The dashed lines are obtained by fitting the exponents in \Cref{eq:scaling}. Loss decreases predictably depending on the model and data size, across 5 orders of magnitudes.}
    \label{fig:scaling_laws_data}
\end{figure*}

\section{Discussion and Future Work} \label{sec:discussion}

Many new questions arise from our work. 

\paragraph{Defining the pre-training universe $\mathcal{G}$}
While it is always possible to merge all available graphs into a single pre-training corpus, it is unclear whether this is optimal, and whether it would yield the best generalization across downstream tasks. This is not only a question of transferability between graphs, but also of how pre-training dynamics are shaped by long-tail distributions over nodes and edges. In particular, rare types in the tail may matter only for a small subset of downstream tasks, it remains unclear whether they should be prioritized during pre-training or reserved for task-specific fine-tuning.
A related challenge concerns the definition of node type and feature groups. We discussed grouping features to ensure applicability to previously unseen features, as done in Tabular FMs. In graphs, the same approach can be also done on the node-type level, on on combination of the two. However, this may come at the cost of expressivity. It would therefore be very interesting to derive concrete rules, potentially as a function of feature, node-type distributions and sparsity with respect to target downstream tasks, for selecting an appropriate level of granularity.

\paragraph{Compute-Optimal Allocation}
In ViTs and LLMs scaling-laws, the context size is fixed and therefore compute is well-approximated by $O(N \cdot D)$ where $N$ represents the number of parameters and $D$ represents the number of training tokens \citep{hoffmann2022training}. This abstraction allows for the derivation of precise compute-optimal frontiers, as the cost per "unit of data" remains constant.
In graph models with neighborhood-based context, the computational
cost per supervised edge depends on the size of the sampled subgraph, which in turn depends
on graph structure. As a result, total pre-training compute cannot be expressed as a universal function of the labeled data and model size alone.
Therefore, any compute-optimal conclusions must be interpreted as conditional on the underlying graph
and sampling distribution, rather than as universal prescriptions.
It would be very valuable to design compute-allocation measures for GFMs, which are transferable across different graphs.

\section{Conclusion}\label{sec:conclusion}
We presented \method{}, an end-to-end recipe for training billion-parameter Graph Foundation Models on billion-scale heterogeneous graphs. We introduced the \method{} Transformer, which combines type-conditioned and type-agnostic attention to efficiently handle the scale, heterogeneity, and degree skewness of real-world graphs, and we established empirical scaling laws showing that loss decreases predictably when model capacity and data are scaled jointly. Across 10 diverse downstream tasks on a billion-scale enterprise graph, we showed that a frozen \method{} Transformer with a simple probing head outperforms all task-specific baselines in every task, with gains of up to 31 PRAUC points, including in few-shot settings on graphs entirely unseen during pre-training. Our results demonstrate that graph foundation models at billion scale can learn transferable representations that generalize across diverse tasks, graph topologies, and evaluation regimes.

\paragraph{Limitations.} The scaling laws are derived from the Enterprise dataset and may not transfer to graphs with fundamentally different degree distributions or feature spaces; in particular, compute-optimal allocation for graph models remains an open problem due to the graph-conditional nature of training cost. The current framework uses masked link prediction as the sole pre-training objective; exploring alternative or complementary self-supervised objectives is an open direction. Finally, the setting of billion-scale graphs composed of very small connected components, as in molecular datasets, is left out of scope. While specialized foundation models exist for such settings, extending \method{} to these domains would be an interesting direction.

\section*{Acknowledgments}
We thank Kfir Amitai for the consistent support and enablement of this research.
We thank Petar Veli\v{c}kovi\'{c}, Amir Bar, Randall Balestriero, Joan Bruna and Uri Sherman for the insightful discussions on LLMs and ViTs as GFMs.
We thank Taco Cohen, Arjun Subramonian, Kaveh Hassani and Saurabh Verma for the insightful discussions and constructive feedback on early versions of this work.

\bibliographystyle{assets/plainnat}
\bibliography{references}

@inproceedings{bechlerspeicher2024graphneuralnetworksuse,
      title={Graph Neural Networks Use Graphs When They Shouldn't},
      author={Maya Bechler-Speicher and Ido Amos and Ran Gilad-Bachrach and Amir Globerson},
      booktitle={Proceedings of the 41st International Conference on Machine Learning},
      year={2024},
      series={ICML'24},
}

@inproceedings{gorishniy2021revisiting,
author = {Gorishniy, Yury and Rubachev, Ivan and Khrulkov, Valentin and Babenko, Artem},
title = {Revisiting deep learning models for tabular data},
year = {2021},
isbn = {9781713845393},
publisher = {Curran Associates Inc.},
address = {Red Hook, NY, USA},
booktitle = {Proceedings of the 35th International Conference on Neural Information Processing Systems},
articleno = {1447},
numpages = {12},
series = {NIPS '21}
}

@inproceedings{somepalli2021saint,
title={{SAINT}: Improved Neural Networks for Tabular Data via Row Attention and Contrastive Pre-Training},
author={Gowthami Somepalli and Avi Schwarzschild and Micah Goldblum and C. Bayan Bruss and Tom Goldstein},
booktitle={NeurIPS 2022 First Table Representation Workshop},
year={2022},
}

@inproceedings{hollmann2022tabpfn,
title={Tab{PFN}: A Transformer That Solves Small Tabular Classification Problems in a Second},
author={Noah Hollmann and Samuel M{\"u}ller and Katharina Eggensperger and Frank Hutter},
booktitle={NeurIPS 2022 First Table Representation Workshop},
year={2022},
}

@inproceedings{shaw2018self,
    title = "Self-Attention with Relative Position Representations",
    author = "Shaw, Peter  and
      Uszkoreit, Jakob  and
      Vaswani, Ashish",
    editor = "Walker, Marilyn  and
      Ji, Heng  and
      Stent, Amanda",
    booktitle = "Proceedings of the 2018 Conference of the North {A}merican Chapter of the Association for Computational Linguistics: Human Language Technologies, Volume 2 (Short Papers)",
    month = jun,
    year = "2018",
    address = "New Orleans, Louisiana",
    publisher = "Association for Computational Linguistics",
    doi = "10.18653/v1/N18-2074",
    pages = "464--468"
}

@misc{finkelshtein2025equivarianceoncerecipegraph,
      title={Equivariance Everywhere All At Once: A Recipe for Graph Foundation Models}, 
      author={Ben Finkelshtein and İsmail İlkan Ceylan and Michael Bronstein and Ron Levie},
      year={2025},
      eprint={2506.14291},
      archivePrefix={arXiv},
      primaryClass={cs.LG},
}

@inproceedings{hamilton2018inductiverepresentationlearninglarge,
author = {Hamilton, William L. and Ying, Rex and Leskovec, Jure},
title = {Inductive representation learning on large graphs},
year = {2017},
isbn = {9781510860964},
publisher = {Curran Associates Inc.},
address = {Red Hook, NY, USA},
booktitle = {Proceedings of the 31st International Conference on Neural Information Processing Systems},
pages = {1025–1035},
numpages = {11},
location = {Long Beach, California, USA},
series = {NIPS'17}
}

@article{kipf2016variational,
  title={Variational Graph Auto-Encoders},
  author={Thomas Kipf and Max Welling},
  journal={ArXiv},
  year={2016},
  volume={abs/1611.07308},
}

@inproceedings{velivckovic2019dgi,
title={Deep Graph Infomax},
author={Petar Veličković and William Fedus and William L. Hamilton and Pietro Liò and Yoshua Bengio and R Devon Hjelm},
booktitle={International Conference on Learning Representations},
year={2019},
}

@inproceedings{hou2022graphmae,
author = {Hou, Zhenyu and Liu, Xiao and Cen, Yukuo and Dong, Yuxiao and Yang, Hongxia and Wang, Chunjie and Tang, Jie},
title = {GraphMAE: Self-Supervised Masked Graph Autoencoders},
year = {2022},
isbn = {9781450393850},
publisher = {Association for Computing Machinery},
address = {New York, NY, USA},
doi = {10.1145/3534678.3539321},
booktitle = {Proceedings of the 28th ACM SIGKDD Conference on Knowledge Discovery and Data Mining},
pages = {594–604},
numpages = {11},
location = {Washington DC, USA},
series = {KDD '22}
}

@inproceedings{hou2022graphmae2,
author = {Hou, Zhenyu and He, Yufei and Cen, Yukuo and Liu, Xiao and Dong, Yuxiao and Kharlamov, Evgeny and Tang, Jie},
title = {GraphMAE2: A Decoding-Enhanced Masked Self-Supervised Graph Learner},
year = {2023},
isbn = {9781450394161},
publisher = {Association for Computing Machinery},
address = {New York, NY, USA},
doi = {10.1145/3543507.3583379},
booktitle = {Proceedings of the ACM Web Conference 2023},
pages = {737–746},
numpages = {10},
location = {Austin, TX, USA},
series = {WWW '23}
}

@inproceedings{you2020graphcl,
author = {You, Yuning and Chen, Tianlong and Sui, Yongduo and Chen, Ting and Wang, Zhangyang and Shen, Yang},
title = {Graph contrastive learning with augmentations},
year = {2020},
isbn = {9781713829546},
publisher = {Curran Associates Inc.},
address = {Red Hook, NY, USA},
booktitle = {Proceedings of the 34th International Conference on Neural Information Processing Systems},
articleno = {488},
numpages = {12},
location = {Vancouver, BC, Canada},
series = {NIPS '20}
}

@misc{zhang2020graphbert,
      title={Graph-Bert: Only Attention is Needed for Learning Graph Representations}, 
      author={Jiawei Zhang and Haopeng Zhang and Congying Xia and Li Sun},
      year={2020},
      eprint={2001.05140},
      archivePrefix={arXiv},
      primaryClass={cs.LG},
}

@inproceedings{vaswani2023attentionneed,
author = {Vaswani, Ashish and Shazeer, Noam and Parmar, Niki and Uszkoreit, Jakob and Jones, Llion and Gomez, Aidan N. and Kaiser, \L{}ukasz and Polosukhin, Illia},
title = {Attention is all you need},
year = {2017},
isbn = {9781510860964},
publisher = {Curran Associates Inc.},
address = {Red Hook, NY, USA},
booktitle = {Proceedings of the 31st International Conference on Neural Information Processing Systems},
pages = {6000–6010},
numpages = {11},
location = {Long Beach, California, USA},
series = {NIPS'17}
}

@misc{eremeev2025turningtabularfoundationmodels,
      title={Turning Tabular Foundation Models into Graph Foundation Models}, 
      author={Dmitry Eremeev and Gleb Bazhenov and Oleg Platonov and Artem Babenko and Liudmila Prokhorenkova},
      year={2025},
      eprint={2508.20906},
      archivePrefix={arXiv},
      primaryClass={cs.LG},

}

@inproceedings{hu2020heterogeneousgraphtransformer,
author = {Hu, Ziniu and Dong, Yuxiao and Wang, Kuansan and Sun, Yizhou},
title = {Heterogeneous Graph Transformer},
year = {2020},
isbn = {9781450370233},
publisher = {Association for Computing Machinery},
address = {New York, NY, USA},
doi = {10.1145/3366423.3380027},
booktitle = {Proceedings of The Web Conference 2020},
pages = {2704–2710},
numpages = {7},
location = {Taipei, Taiwan},
series = {WWW '20}
}

@inproceedings{kipf2017semisupervisedclassificationgraphconvolutional,
      title={Semi-Supervised Classification with Graph Convolutional Networks},
      author={Thomas N. Kipf and Max Welling},
      booktitle={International Conference on Learning Representations},
      year={2017},
}

@inproceedings{galkin2024foundationmodelsknowledgegraph,
      title={Towards Foundation Models for Knowledge Graph Reasoning},
      author={Mikhail Galkin and Xinyu Yuan and Hesham Mostafa and Jian Tang and Zhaocheng Zhu},
      booktitle={International Conference on Learning Representations},
      year={2024},
}

@misc{shoghi2024moleculesmaterialspretraininglarge,
      title={From Molecules to Materials: Pre-training Large Generalizable Models for Atomic Property Prediction}, 
      author={Nima Shoghi and Adeesh Kolluru and John R. Kitchin and Zachary W. Ulissi and C. Lawrence Zitnick and Brandon M. Wood},
      year={2024},
      eprint={2310.16802},
      archivePrefix={arXiv},
      primaryClass={cs.LG},

}

@misc{liu2024alltraininggraphmodel,
      title={One for All: Towards Training One Graph Model for All Classification Tasks}, 
      author={Hao Liu and Jiarui Feng and Lecheng Kong and Ningyue Liang and Dacheng Tao and Yixin Chen and Muhan Zhang},
      year={2024},
      eprint={2310.00149},
      archivePrefix={arXiv},
      primaryClass={cs.LG},
}

@inproceedings{mao2024positiongraphfoundationmodels,
author = {Mao, Haitao and Chen, Zhikai and Tang, Wenzhuo and Zhao, Jianan and Ma, Yao and Zhao, Tong and Shah, Neil and Galkin, Mikhail and Tang, Jiliang},
title = {Position: graph foundation models are already here},
year = {2024},
publisher = {JMLR.org},
booktitle = {Proceedings of the 41st International Conference on Machine Learning},
articleno = {1410},
numpages = {23},
location = {Vienna, Austria},
series = {ICML'24}
}

@misc{ranjan2025relationaltransformerzeroshotfoundation,
      title={Relational Transformer: Toward Zero-Shot Foundation Models for Relational Data}, 
      author={Rishabh Ranjan and Valter Hudovernik and Mark Znidar and Charilaos Kanatsoulis and Roshan Upendra and Mahmoud Mohammadi and Joe Meyer and Tom Palczewski and Carlos Guestrin and Jure Leskovec},
      year={2025},
      eprint={2510.06377},
      archivePrefix={arXiv},
      primaryClass={cs.LG},

}

@inproceedings{gilmer2017neural,
author = {Gilmer, Justin and Schoenholz, Samuel S. and Riley, Patrick F. and Vinyals, Oriol and Dahl, George E.},
title = {Neural message passing for Quantum chemistry},
year = {2017},
publisher = {JMLR.org},
booktitle = {Proceedings of the 34th International Conference on Machine Learning - Volume 70},
pages = {1263–1272},
numpages = {10},
location = {Sydney, NSW, Australia},
series = {ICML'17}
}

@article{Bahri_2024,
author = {Yasaman Bahri  and Ethan Dyer  and Jared Kaplan  and Jaehoon Lee  and Utkarsh Sharma },
title = {Explaining neural scaling laws},
journal = {Proceedings of the National Academy of Sciences},
volume = {121},
number = {27},
pages = {e2311878121},
year = {2024},
doi = {10.1073/pnas.2311878121},
eprint = {https://www.pnas.org/doi/pdf/10.1073/pnas.2311878121}
}

@misc{tay2022scalinglawsvsmodel,
      title={Scaling Laws vs Model Architectures: How does Inductive Bias Influence Scaling?}, 
      author={Yi Tay and Mostafa Dehghani and Samira Abnar and Hyung Won Chung and William Fedus and Jinfeng Rao and Sharan Narang and Vinh Q. Tran and Dani Yogatama and Donald Metzler},
      year={2022},
      eprint={2207.10551},
      archivePrefix={arXiv},
      primaryClass={cs.LG},

}

@misc{brehmer2025doesequivariancematterscale,
      title={Does equivariance matter at scale?}, 
      author={Johann Brehmer and Sönke Behrends and Pim de Haan and Taco Cohen},
      year={2025},
      eprint={2410.23179},
      archivePrefix={arXiv},
      primaryClass={cs.LG},
}

@article{traag2019leiden,
  title        = {From Louvain to Leiden: guaranteeing well-connected communities},
  author       = {Traag, Vincent A and Waltman, Ludo and van Eck, Nees Jan},
  journal      = {Scientific Reports},
  volume       = {9},
  number       = {1},
  pages        = {5233},
  year         = {2019},
  publisher    = {Nature Publishing Group},
  doi          = {10.1038/s41598-019-41695-z}
}

@inproceedings{Chiang_2019,
author = {Chiang, Wei-Lin and Liu, Xuanqing and Si, Si and Li, Yang and Bengio, Samy and Hsieh, Cho-Jui},
title = {Cluster-GCN: An Efficient Algorithm for Training Deep and Large Graph Convolutional Networks},
year = {2019},
isbn = {9781450362016},
publisher = {Association for Computing Machinery},
address = {New York, NY, USA},
doi = {10.1145/3292500.3330925},
booktitle = {Proceedings of the 25th ACM SIGKDD International Conference on Knowledge Discovery \& Data Mining},
pages = {257–266},
numpages = {10},
location = {Anchorage, AK, USA},
series = {KDD '19}
}

@misc{wang2025learninggraphquantizedtokenizers,
      title={Learning Graph Quantized Tokenizers}, 
      author={Limei Wang and Kaveh Hassani and Si Zhang and Dongqi Fu and Baichuan Yuan and Weilin Cong and Zhigang Hua and Hao Wu and Ning Yao and Bo Long},
      year={2025},
      eprint={2410.13798},
      archivePrefix={arXiv},
      primaryClass={cs.NE},
}

@inproceedings{10.1145/1553374.1553380,
author = {Bengio, Yoshua and Louradour, J\'{e}r\^{o}me and Collobert, Ronan and Weston, Jason},
title = {Curriculum learning},
year = {2009},
isbn = {9781605585161},
publisher = {Association for Computing Machinery},
address = {New York, NY, USA},
doi = {10.1145/1553374.1553380},
booktitle = {Proceedings of the 26th Annual International Conference on Machine Learning},
pages = {41–48},
numpages = {8},
location = {Montreal, Quebec, Canada},
series = {ICML '09}
}

@misc{keskar2017largebatchtrainingdeeplearning,
      title={On Large-Batch Training for Deep Learning: Generalization Gap and Sharp Minima}, 
      author={Nitish Shirish Keskar and Dheevatsa Mudigere and Jorge Nocedal and Mikhail Smelyanskiy and Ping Tak Peter Tang},
      year={2017},
      eprint={1609.04836},
      archivePrefix={arXiv},
      primaryClass={cs.LG},
}

@inproceedings{zeng2020graphsaintgraphsamplingbased,
      title={GraphSAINT: Graph Sampling Based Inductive Learning Method},
      author={Hanqing Zeng and Hongkuan Zhou and Ajitesh Srivastava and Rajgopal Kannan and Viktor Prasanna},
      booktitle={International Conference on Learning Representations},
      year={2020},
}

@inproceedings{tai2020exbert,
    title = "ex{BERT}: Extending Pre-trained Models with Domain-specific Vocabulary Under Constrained Training Resources",
    author = "Tai, Wen  and
      Kung, H. T.  and
      Dong, Xin  and
      Comiter, Marcus  and
      Kuo, Chang-Fu",
    editor = "Cohn, Trevor  and
      He, Yulan  and
      Liu, Yang",
    booktitle = "Findings of the Association for Computational Linguistics: EMNLP 2020",
    month = nov,
    year = "2020",
    address = "Online",
    publisher = "Association for Computational Linguistics",
    doi = "10.18653/v1/2020.findings-emnlp.129",
    pages = "1433--1439"
}

@inproceedings{hu2021lora,
title={Lo{RA}: Low-Rank Adaptation of Large Language Models},
author={Edward J Hu and yelong shen and Phillip Wallis and Zeyuan Allen-Zhu and Yuanzhi Li and Shean Wang and Lu Wang and Weizhu Chen},
booktitle={International Conference on Learning Representations},
year={2022},
}

@misc{bechlerspeicher2025positiongraphlearninglose,
      title={Position: Graph Learning Will Lose Relevance Due To Poor Benchmarks}, 
      author={Maya Bechler-Speicher and Ben Finkelshtein and Fabrizio Frasca and Luis Müller and Jan Tönshoff and Antoine Siraudin and Viktor Zaverkin and Michael M. Bronstein and Mathias Niepert and Bryan Perozzi and Mikhail Galkin and Christopher Morris},
      year={2025},
      eprint={2502.14546},
      archivePrefix={arXiv},
      primaryClass={cs.LG},
}

@article{Veli_kovi__2023,
title = {Everything is connected: Graph neural networks},
journal = {Current Opinion in Structural Biology},
volume = {79},
pages = {102538},
year = {2023},
issn = {0959-440X},
doi = {https://doi.org/10.1016/j.sbi.2023.102538},
author = {Petar Veličković}
}

@inproceedings{devlin2019bert,
    title = "{BERT}: Pre-training of Deep Bidirectional Transformers for Language Understanding",
    author = "Devlin, Jacob  and
      Chang, Ming-Wei  and
      Lee, Kenton  and
      Toutanova, Kristina",
    editor = "Burstein, Jill  and
      Doran, Christy  and
      Solorio, Thamar",
    booktitle = "Proceedings of the 2019 Conference of the North {A}merican Chapter of the Association for Computational Linguistics: Human Language Technologies, Volume 1 (Long and Short Papers)",
    month = jun,
    year = "2019",
    address = "Minneapolis, Minnesota",
    publisher = "Association for Computational Linguistics",
    doi = "10.18653/v1/N19-1423",
    pages = "4171--4186"
}

@misc{kaplan2020scaling,
      title={Scaling Laws for Neural Language Models}, 
      author={Jared Kaplan and Sam McCandlish and Tom Henighan and Tom B. Brown and Benjamin Chess and Rewon Child and Scott Gray and Alec Radford and Jeffrey Wu and Dario Amodei},
      year={2020},
      eprint={2001.08361},
      archivePrefix={arXiv},
      primaryClass={cs.LG},
}

@inproceedings{dosovitskiy2021an,
title={An Image is Worth 16x16 Words: Transformers for Image Recognition at Scale},
author={Alexey Dosovitskiy and Lucas Beyer and Alexander Kolesnikov and Dirk Weissenborn and Xiaohua Zhai and Thomas Unterthiner and Mostafa Dehghani and Matthias Minderer and Georg Heigold and Sylvain Gelly and Jakob Uszkoreit and Neil Houlsby},
booktitle={International Conference on Learning Representations},
year={2021},
}

@inproceedings{brown2020language,
 author = {Brown, Tom and Mann, Benjamin and Ryder, Nick and Subbiah, Melanie and Kaplan, Jared D and Dhariwal, Prafulla and Neelakantan, Arvind and Shyam, Pranav and Sastry, Girish and Askell, Amanda and Agarwal, Sandhini and Herbert-Voss, Ariel and Krueger, Gretchen and Henighan, Tom and Child, Rewon and Ramesh, Aditya and Ziegler, Daniel and Wu, Jeffrey and Winter, Clemens and Hesse, Chris and Chen, Mark and Sigler, Eric and Litwin, Mateusz and Gray, Scott and Chess, Benjamin and Clark, Jack and Berner, Christopher and McCandlish, Sam and Radford, Alec and Sutskever, Ilya and Amodei, Dario},
 booktitle = {Advances in Neural Information Processing Systems},
 editor = {H. Larochelle and M. Ranzato and R. Hadsell and M.F. Balcan and H. Lin},
 pages = {1877--1901},
 publisher = {Curran Associates, Inc.},
 title = {Language Models are Few-Shot Learners},
 volume = {33},
 year = {2020}
}

@inproceedings{
fatemi2023talklikegraphencoding,
title={Talk like a Graph: Encoding Graphs for Large Language Models},
author={Bahare Fatemi and Jonathan Halcrow and Bryan Perozzi},
booktitle={The Twelfth International Conference on Learning Representations},
year={2024},
}

@misc{bronstein2021geometricdeeplearninggrids,
      title={Geometric Deep Learning: Grids, Groups, Graphs, Geodesics, and Gauges}, 
      author={Michael M. Bronstein and Joan Bruna and Taco Cohen and Petar Veličković},
      year={2021},
      eprint={2104.13478},
      archivePrefix={arXiv},
      primaryClass={cs.LG},
      url={https://arxiv.org/abs/2104.13478}, 
}

@inproceedings{velickovic2018graph,
title={Graph Attention Networks},
author={Petar Veličković and Guillem Cucurull and Arantxa Casanova and Adriana Romero and Pietro Liò and Yoshua Bengio},
booktitle={International Conference on Learning Representations},
year={2018},
}

@inproceedings{sanford2024transformersparallelcomputationlogarithmic,
author = {Sanford, Clayton and Hsu, Daniel and Telgarsky, Matus},
title = {Transformers, parallel computation, and logarithmic depth},
year = {2024},
publisher = {JMLR.org},
booktitle = {Proceedings of the 41st International Conference on Machine Learning},
articleno = {1763},
numpages = {52},
location = {Vienna, Austria},
series = {ICML'24}
}

@misc{sanford2023representationalstrengthslimitationstransformers,
      title={Representational Strengths and Limitations of Transformers}, 
      author={Clayton Sanford and Daniel Hsu and Matus Telgarsky},
      year={2023},
      eprint={2306.02896},
      archivePrefix={arXiv},
      primaryClass={cs.LG},
}

@misc{yehudai2025depthwidthtradeoffsalgorithmicreasoning,
      title={Depth-Width tradeoffs in Algorithmic Reasoning of Graph Tasks with Transformers}, 
      author={Gilad Yehudai and Clayton Sanford and Maya Bechler-Speicher and Orr Fischer and Ran Gilad-Bachrach and Amir Globerson},
      year={2025},
      eprint={2503.01805},
      archivePrefix={arXiv},
      primaryClass={cs.LG},
}

@inproceedings{fu2020magnn,
author = {Fu, Xinyu and Zhang, Jiani and Meng, Ziqiao and King, Irwin},
title = {MAGNN: Metapath Aggregated Graph Neural Network for Heterogeneous Graph Embedding},
year = {2020},
isbn = {9781450370233},
publisher = {Association for Computing Machinery},
address = {New York, NY, USA},
doi = {10.1145/3366423.3380297},
booktitle = {Proceedings of The Web Conference 2020},
pages = {2331–2341},
numpages = {11},
location = {Taipei, Taiwan},
series = {WWW '20}
}

@misc{radford2021learningtransferablevisualmodels,
      title={Learning Transferable Visual Models From Natural Language Supervision}, 
      author={Alec Radford and Jong Wook Kim and Chris Hallacy and Aditya Ramesh and Gabriel Goh and Sandhini Agarwal and Girish Sastry and Amanda Askell and Pamela Mishkin and Jack Clark and Gretchen Krueger and Ilya Sutskever},
      year={2021},
      eprint={2103.00020},
      archivePrefix={arXiv},
      primaryClass={cs.CV},
}

@inproceedings{luo2023closerlookfewshotclassification,
author = {Luo, Xu and Wu, Hao and Zhang, Ji and Gao, Lianli and Xu, Jing and Song, Jingkuan},
title = {A closer look at few-shot classification again},
year = {2023},
publisher = {JMLR.org},
booktitle = {Proceedings of the 40th International Conference on Machine Learning},
articleno = {960},
numpages = {21},
location = {Honolulu, Hawaii, USA},
series = {ICML'23}
}

@inproceedings{wang2021heterogeneousgraphattentionnetwork,
author = {Wang, Xiao and Ji, Houye and Shi, Chuan and Wang, Bai and Ye, Yanfang and Cui, Peng and Yu, Philip S},
title = {Heterogeneous Graph Attention Network},
year = {2019},
isbn = {9781450366748},
publisher = {Association for Computing Machinery},
address = {New York, NY, USA},
doi = {10.1145/3308558.3313562},
booktitle = {The World Wide Web Conference},
pages = {2022–2032},
numpages = {11},
location = {San Francisco, CA, USA},
series = {WWW '19}
}

@book{wasserman1994social,
  title     = {Social Network Analysis: Methods and Applications},
  author    = {Stanley Wasserman, Katherine Faust},
  publisher = {Cambridge University Press},
  year      = {1994},
  address   = {Cambridge, UK},
  isbn      = {9780521387071}
}

@article{Shehzad_2026,
  title={Graph Transformers: A Survey},
  author={Ahsan Shehzad and Feng Xia and Shagufta Abid and Ciyuan Peng and Shuo Yu and Dongyu Zhang and Karin M. Verspoor},
  journal={IEEE transactions on neural networks and learning systems},
  year={2024},
  volume={PP},
}

@inproceedings{ying2021transformer,
author = {Ying, Chengxuan and Cai, Tianle and Luo, Shengjie and Zheng, Shuxin and Ke, Guolin and He, Di and Shen, Yanming and Liu, Tie-Yan},
title = {Do transformers really perform bad for graph representation?},
year = {2021},
isbn = {9781713845393},
publisher = {Curran Associates Inc.},
address = {Red Hook, NY, USA},
booktitle = {Proceedings of the 35th International Conference on Neural Information Processing Systems},
articleno = {2212},
numpages = {12},
series = {NIPS '21}
}

@inproceedings{dwivedi2021graphtransformer,
      title={A Generalization of Transformer Networks to Graphs},
      author={Vijay Prakash Dwivedi and Xavier Bresson},
      booktitle={AAAI 2021 Workshop on Deep Learning on Graphs: Methods and Applications},
      year={2021},
}

@inproceedings{kreuzer2021rethinking,
author = {Kreuzer, Devin and Beaini, Dominique and Hamilton, William L. and L\'{e}tourneau, Vincent and Tossou, Prudencio},
title = {Rethinking graph transformers with spectral attention},
year = {2021},
isbn = {9781713845393},
publisher = {Curran Associates Inc.},
address = {Red Hook, NY, USA},
booktitle = {Proceedings of the 35th International Conference on Neural Information Processing Systems},
articleno = {1654},
numpages = {12},
series = {NIPS '21}
}

@inproceedings{rampasek2022gps,
author = {Ramp\'{a}\v{s}ek, Ladislav and Galkin, Mikhail and Dwivedi, Vijay Prakash and Luu, Anh Tuan and Wolf, Guy and Beaini, Dominique},
title = {Recipe for a general, powerful, scalable graph transformer},
year = {2022},
isbn = {9781713871088},
publisher = {Curran Associates Inc.},
address = {Red Hook, NY, USA},
booktitle = {Proceedings of the 36th International Conference on Neural Information Processing Systems},
articleno = {1054},
numpages = {15},
location = {New Orleans, LA, USA},
series = {NIPS '22}
}

@article{bommasani2021foundation,
title={On the Opportunities and Risks of Foundation Models},
author={Rishi Bommasani and Drew A. Hudson and Ehsan Adeli and Russ Altman and Simran Arora and Sydney von Arx and Michael S. Bernstein and Jeannette Bohg and Antoine Bosselut and Emma Brunskill and Erik Brynjolfsson and S. Buch and Dallas Card and Rodrigo Castellon and Niladri S. Chatterji and Annie S. Chen and Kathleen A. Creel and Jared Davis and Dora Demszky and Chris Donahue and Moussa Doumbouya and Esin Durmus and Stefano Ermon and John Etchemendy and Kawin Ethayarajh and Li Fei-Fei and Chelsea Finn and Trevor Gale and Lauren E. Gillespie and Karan Goel and Noah D. Goodman and Shelby Grossman and Neel Guha and Tatsunori Hashimoto and Peter Henderson and John Hewitt and Daniel E. Ho and Jenny Hong and Kyle Hsu and Jing Huang and Thomas F. Icard and Saahil Jain and Dan Jurafsky and Pratyusha Kalluri and Siddharth Karamcheti and Geoff Keeling and Fereshte Khani and O. Khattab and Pang Wei Koh and Mark S. Krass and Ranjay Krishna and Rohith Kuditipudi and Ananya Kumar and Faisal Ladhak and Mina Lee and Tony Lee and Jure Leskovec and Isabelle Levent and Xiang Lisa Li and Xuechen Li and Tengyu Ma and Ali Malik and Christopher D. Manning and Suvir P. Mirchandani and Eric Mitchell and Zanele Munyikwa and Suraj Nair and Avanika Narayan and Deepak Narayanan and Benjamin Newman and Allen Nie and Juan Carlos Niebles and Hamed Nilforoshan and J. F. Nyarko and Giray Ogut and Laurel Orr and Isabel Papadimitriou and Joon Sung Park and Chris Piech and Eva Portelance and Christopher Potts and Aditi Raghunathan and Robert Reich and Hongyu Ren and Frieda Rong and Yusuf H. Roohani and Camilo Ruiz and Jack Ryan and Christopher R'e and Dorsa Sadigh and Shiori Sagawa and Keshav Santhanam and Andy Shih and Krishna Parasuram Srinivasan and Alex Tamkin and Rohan Taori and Armin W. Thomas and Florian Tram{\`e}r and Rose E. Wang and William Wang and Bohan Wu and Jiajun Wu and Yuhuai Wu and Sang Michael Xie and Michihiro Yasunaga and Jiaxuan You and Matei A. Zaharia and Michael Zhang and Tianyi Zhang and Xikun Zhang and Yuhui Zhang and Lucia Zheng and Kaitlyn Zhou and Percy Liang},
journal={ArXiv},
year={2021},
}

@inproceedings{hu2020strategies,
title={Strategies for Pre-training Graph Neural Networks},
author={Weihua Hu* and Bowen Liu* and Joseph Gomes and Marinka Zitnik and Percy Liang and Vijay Pande and Jure Leskovec},
booktitle={International Conference on Learning Representations},
year={2020},
}

@misc{wang2025graphfoundationmodelscomprehensive,
      title={Graph Foundation Models: A Comprehensive Survey}, 
      author={Zehong Wang and Zheyuan Liu and Tianyi Ma and Jiazheng Li and Zheyuan Zhang and Xingbo Fu and Yiyang Li and Zhengqing Yuan and Wei Song and Yijun Ma and Qingkai Zeng and Xiusi Chen and Jianan Zhao and Jundong Li and Meng Jiang and Pietro Lio and Nitesh Chawla and Chuxu Zhang and Yanfang Ye},
      year={2025},
      eprint={2505.15116},
      archivePrefix={arXiv},
      primaryClass={cs.LG},
}

@inproceedings{GraphAny2024,
title={Fully-inductive Node Classification on Arbitrary Graphs},
author={Jianan Zhao and Zhaocheng Zhu and Mikhail Galkin and Hesham Mostafa and Michael M. Bronstein and Jian Tang},
booktitle={The Thirteenth International Conference on Learning Representations},
year={2025},
}

@inproceedings{hoffmann2022training,
author = {Hoffmann, Jordan and Borgeaud, Sebastian and Mensch, Arthur and Buchatskaya, Elena and Cai, Trevor and Rutherford, Eliza and de Las Casas, Diego and Hendricks, Lisa Anne and Welbl, Johannes and Clark, Aidan and Hennigan, Tom and Noland, Eric and Millican, Katie and van den Driessche, George and Damoc, Bogdan and Guy, Aurelia and Osindero, Simon and Simonyan, Karen and Elsen, Erich and Vinyals, Oriol and Rae, Jack W. and Sifre, Laurent},
title = {Training compute-optimal large language models},
year = {2022},
isbn = {9781713871088},
publisher = {Curran Associates Inc.},
address = {Red Hook, NY, USA},
booktitle = {Proceedings of the 36th International Conference on Neural Information Processing Systems},
articleno = {2176},
numpages = {15},
location = {New Orleans, LA, USA},
series = {NIPS '22}
}

\clearpage
\appendix
\section{Proof of Theorem 4.1}
Here we prove Theorem 4.1.

\begin{proof}[Proof of Theorem 4.1]
Fix a node $v\in\mathcal{V}$ and let $\mathcal{N}^{\mathrm{taa}}_v$ be its TAA neighborhood.
For any set of edge types $S\subseteq\mathcal{T}_E$, let $\mathcal{N}^S_v\subseteq \mathcal{N}^{\mathrm{tca}}_v$
denote the $S$-masked neighborhood. In particular, with $\mathcal{N}^{\mathrm{tca}}_v=\mathcal{N}^{\mathrm{taa}}_v$
and $\mathcal{S}=\big\{\{r^\star\},\{r'\}\big\}$ we have two disjoint masked neighborhoods
$\mathcal{N}^{\{r^\star\}}_v$ and $\mathcal{N}^{\{r'\}}_v$.

Consider a family of graphs where all node types are identical (so the type projection
$\mathbf{W}^{(\ell)}_{\tau(v)}$ can be taken as the identity without loss of generality),
where edges carry no features ($\mathbf{g}_{uv} = \mathbf{0}$ for all edges),
and where each neighbor $u\in\mathcal{N}^{\mathrm{taa}}_v$ has a one-dimensional feature
$x_u\in\{0,1\}$, i.e., $\mathbf{h}^{(\ell)}_u=x_u$.

Define the two scalars
\[
a(v)\ :=\ \frac{1}{|\mathcal{N}^{\{r^\star\}}_v|}\sum_{u\in \mathcal{N}^{\{r^\star\}}_v} x_u,
\qquad
b(v)\ :=\ \frac{1}{|\mathcal{N}^{\mathrm{taa}}_v|}\sum_{u\in \mathcal{N}^{\mathrm{taa}}_v} x_u.
\]

Fix any $\varepsilon\in(0,\tfrac14)$ and define the continuous ``soft-threshold'' function
\[
s_\varepsilon(t)\ :=\ \mathrm{clip}\!\left(\frac{t-(\tfrac12-\varepsilon)}{2\varepsilon},\,0,\,1\right),
\qquad
\mathrm{clip}(z,0,1):=\min(1,\max(0,z)).
\]
Define the target
\[
\tilde f(v)\ :=\ s_\varepsilon(a(v))\ +\ s_\varepsilon(b(v)).
\].

\subsubsection*{$\tilde f\in\mathcal{F}_{\mathrm{GraphBFF}}$}

\emph{(i) TAA computes $b(v)$.}
Choose the (shared) TAA attention parameters so that all attention logits are equal.
Set $\mathbf{W}^{(\ell)}_{Q}=\mathbf{0}$, $\mathbf{W}^{(\ell)}_{K}=\mathbf{0}$, and $\bar{\mathbf{a}}^{(\ell)}=\mathbf{0}$,
so that for every $u\in\mathcal{N}^{\mathrm{taa}}_v$ the score
$\frac{(\mathbf{q}^{(\ell)}_v)^\top \mathbf{k}^{(\ell)}_u}{\sqrt{d_h}}$ is constant and thus
\[
\beta^{(\ell)}_{uv}=\frac{1}{|\mathcal{N}^{\mathrm{taa}}_v|}.
\]
Set $\mathbf{W}^{(\ell)}_{V}$ so that the first coordinate of $\mathbf{v}^{(\ell)}_u$ equals $x_u$.
Then the first coordinate of $\mathbf{h}^{(\ell,\mathrm{taa})}_v$ is exactly
\[
\sum_{u\in\mathcal{N}^{\mathrm{taa}}_v}\beta^{(\ell)}_{uv}\,x_u \;=\; b(v).
\]

\emph{(ii) TCA computes $a(v)$, and uses $\sum$ over sets.}
For the set $S=\{r^\star\}$, set $\mathbf{W}^{(\{r^\star\},\ell)}_{Q}=\mathbf{0}$,
$\mathbf{W}^{(\{r^\star\},\ell)}_{K}=\mathbf{0}$, and $\mathbf{a}^{(\{r^\star\},\ell)}=\mathbf{0}$ so the softmax is uniform on $\mathcal{N}^{\{r^\star\}}_v$:
\[
\alpha^{(\{r^\star\},\ell)}_{uv}=\frac{1}{|\mathcal{N}^{\{r^\star\}}_v|}.
\]
Choose $\mathbf{W}^{(\{r^\star\},\ell)}_{V}$ so that the first coordinate of
$\mathbf{v}^{(\{r^\star\},\ell)}_u$ equals $x_u$. Then the first coordinate of
$\mathbf{h}^{(\ell,\{r^\star\})}_v$ equals $a(v)$.

For $S=\{r'\}$, set $\mathbf{W}^{(\{r'\},\ell)}_{V}=\mathbf{0}$ so that
$\mathbf{h}^{(\ell,\{r'\})}_v=\mathbf{0}$.
Since in this variant the set-aggregation is \emph{sum},
\[
\mathbf{h}^{(\ell,\mathrm{tca})}_v
=\sum_{S\in\mathcal{S}} \mathbf{h}^{(\ell,S)}_v
=\mathbf{h}^{(\ell,\{r^\star\})}_v+\mathbf{h}^{(\ell,\{r'\})}_v
=\mathbf{h}^{(\ell,\{r^\star\})}_v,
\]
so $\mathbf{h}^{(\ell,\mathrm{tca})}_v$ contains $a(v)$ in a fixed coordinate.

\emph{(iii) A ReLU-MLP readout computes $\tilde f(v)$ exactly from $(a(v),b(v))$.}
Consider the 2-layer ReLU network (one hidden layer) taking input
$x=\begin{bmatrix}a\\ b\end{bmatrix}$ and producing output $\tilde f$:
\[
h=\mathrm{ReLU}(W_1 x+b_1),\qquad \tilde f = W_2 h + b_2,
\]
with
\[
W_1=
\begin{bmatrix}
\frac{1}{2\varepsilon} & 0\\
\frac{1}{2\varepsilon} & 0\\
0 & \frac{1}{2\varepsilon}\\
0 & \frac{1}{2\varepsilon}
\end{bmatrix},
\quad
b_1=
\begin{bmatrix}
\frac12-\frac{1}{4\varepsilon}\\[2pt]
-\frac12-\frac{1}{4\varepsilon}\\[2pt]
\frac12-\frac{1}{4\varepsilon}\\[2pt]
-\frac12-\frac{1}{4\varepsilon}
\end{bmatrix},
\quad
W_2=\begin{bmatrix}1&-1&1&-1\end{bmatrix},
\quad
b_2=0.
\]
This implements
\[
\tilde f \;=\;\big(\mathrm{ReLU}(z(a))-\mathrm{ReLU}(z(a)-1)\big)\;+\;\big(\mathrm{ReLU}(z(b))-\mathrm{ReLU}(z(b)-1)\big)
\;=\; s_\varepsilon(a)+s_\varepsilon(b),
\]
where $z(t)=\frac{t-(\frac12-\varepsilon)}{2\varepsilon}$ and we use the identity
$\mathrm{clip}(z,0,1)=\mathrm{ReLU}(z)-\mathrm{ReLU}(z-1)$.
Thus, choosing $\Phi^{(\ell)}$ to contain the above ReLU-MLP yields $\tilde f(v)$.
To ensure that the subsequent LayerNorm does not distort the constructed outputs, we embed the scalars $a(v)$ and $b(v)$ alongside $d_\ell - 2$ constant dummy coordinates set to a large fixed value $M$. For sufficiently large $M$, the per-node mean and variance are dominated by the constant coordinates and become nearly identical across all graphs in the construction, so LayerNorm acts as an approximately affine map that preserves the discriminability of the first two coordinates.
Hence $\tilde f\in\mathcal{F}_{\mathrm{GraphBFF}}$.

\subsubsection*{$\tilde f\notin\mathcal{F}_{\mathrm{TAA}}$}

We build two graphs that any TAA-only model must map to the same output at $v$, yet $\tilde f(v)$ differs.
Let $\mathcal{N}^{\mathrm{taa}}_v=\{u_1,u_2\}$ and set $(x_{u_1},x_{u_2})=(1,0)$.
Define two graphs $\mathcal{G}_A,\mathcal{G}_B$ that are identical except for swapping edge-type labels:
\[
\phi((u_1,v))=r^\star,\ \phi((u_2,v))=r' \ \text{ in }\mathcal{G}_A,
\qquad
\phi((u_1,v))=r',\ \phi((u_2,v))=r^\star \ \text{ in }\mathcal{G}_B.
\]
Then $b(v)=\tfrac12$ in both graphs, while $a(v)=1$ in $\mathcal{G}_A$ and $a(v)=0$ in $\mathcal{G}_B$.
Since $\varepsilon<\tfrac12$, we have $s_\varepsilon(1)=1$ and $s_\varepsilon(0)=0$, and also $s_\varepsilon(\tfrac12)=\tfrac12$.
Therefore
\[
\tilde f_{\mathcal{G}_A}(v)=1+\tfrac12=\tfrac32,\qquad
\tilde f_{\mathcal{G}_B}(v)=0+\tfrac12=\tfrac12,
\]
so $\tilde f_{\mathcal{G}_A}(v)\neq \tilde f_{\mathcal{G}_B}(v)$.

However, under the stated TAA parameter sharing
$\mathbf{W}^{(\ell)}_{Q,r}=\mathbf{W}^{(\ell)}_{Q}$,
$\mathbf{W}^{(\ell)}_{K,r}=\mathbf{W}^{(\ell)}_{K}$,
$\mathbf{W}^{(\ell)}_{V,r}=\mathbf{W}^{(\ell)}_{V}$ for all $r$,
the TAA computation at $v$ depends only on the multiset of neighbor representations
$\{\widehat{\mathbf{h}}^{(\ell)}_{u}:u\in\mathcal{N}^{\mathrm{taa}}_v\}$ and is invariant to swapping edge-type labels.
Since this multiset is identical in $\mathcal{G}_A$ and $\mathcal{G}_B$, the TAA-only embedding at $v$ is identical,
and any MLP readout must output the same value on both graphs, a contradiction.
Therefore $\tilde f\notin\mathcal{F}_{\mathrm{TAA}}$.

\subsubsection*{$\tilde f\notin\mathcal{F}_{\mathrm{TCA}}$}

We build two graphs that any TCA-only model must map to the same output at $v$, yet $\tilde f(v)$ differs,
using only that each set-specific TCA vector $\mathbf{h}^{(\ell,S)}_v$ is a softmax-weighted average over
$\{\mathbf{h}^{(\ell)}_u:u\in\mathcal{N}^S_v\}$.

Construct two graphs $\mathcal{G}_C,\mathcal{G}_D$ such that
\[
x_u=1\ \ \forall u\in\mathcal{N}^{\{r^\star\}}_v,
\qquad
x_u=0\ \ \forall u\in\mathcal{N}^{\{r'\}}_v,
\]
and the only difference is the neighborhood cardinalities:
\[
|\mathcal{N}^{\{r^\star\}}_v|=1,\ |\mathcal{N}^{\{r'\}}_v|=3 \ \text{ in }\mathcal{G}_C,
\qquad
|\mathcal{N}^{\{r^\star\}}_v|=3,\ |\mathcal{N}^{\{r'\}}_v|=1 \ \text{ in }\mathcal{G}_D.
\]
Then $a(v)=1$ in both graphs, but
\[
b_{\mathcal{G}_C}(v)=\tfrac14,\qquad b_{\mathcal{G}_D}(v)=\tfrac34.
\]
With $\varepsilon\in(0,\tfrac14)$, we have $s_\varepsilon(\tfrac14)=0$ and $s_\varepsilon(\tfrac34)=1$, and $s_\varepsilon(1)=1$.
Thus
\[
\tilde f_{\mathcal{G}_C}(v)=s_\varepsilon(1)+s_\varepsilon(\tfrac14)=1+0=1,
\qquad
\tilde f_{\mathcal{G}_D}(v)=s_\varepsilon(1)+s_\varepsilon(\tfrac34)=1+1=2,
\]
so $\tilde f_{\mathcal{G}_C}(v)\neq \tilde f_{\mathcal{G}_D}(v)$.

Now consider any TCA-only instantiation.
For $S=\{r^\star\}$, all nodes in $\mathcal{N}^{\{r^\star\}}_v$ have the same input representation $\mathbf{h}^{(\ell)}_u=1$,
hence the projected values $\mathbf{v}^{(S,\ell)}_u=\mathbf{W}^{(S,\ell)}_{V}\mathbf{h}^{(\ell)}_u$ are identical across
$u\in\mathcal{N}^{\{r^\star\}}_v$, and similarly for $S=\{r'\}$ (with $\mathbf{h}^{(\ell)}_u=0$).
Therefore, in each set $S$, the attention output
\[
\mathbf{h}^{(\ell,S)}_v=\sum_{u\in\mathcal{N}^S_v}\alpha^{(S,\ell)}_{uv}\,\mathbf{v}^{(S,\ell)}_u
\]
is the same in $\mathcal{G}_C$ and $\mathcal{G}_D$, because it is a convex combination of identical vectors
(and thus does not depend on $|\mathcal{N}^S_v|$).
Consequently, for every $S\in\mathcal{S}$, the vectors $\mathbf{h}^{(\ell,S)}_v$ match between $\mathcal{G}_C$ and $\mathcal{G}_D$,
and so does their \emph{sum}
\[
\mathbf{h}^{(\ell,\mathrm{tca})}_v=\sum_{S\in\mathcal{S}}\mathbf{h}^{(\ell,S)}_v.
\]
Hence any TCA-only model (followed by any MLP readout) must output the same value at $v$ on $\mathcal{G}_C$ and $\mathcal{G}_D$,
contradicting $\tilde f_{\mathcal{G}_C}(v)\neq \tilde f_{\mathcal{G}_D}(v)$.
Therefore $\tilde f\notin\mathcal{F}_{\mathrm{TCA}}$.
\end{proof}

\section{KL-Batching Additional Formulation}
Here we formulate and describe in details \textit{KL-Batching}.
As loading data from storage into memory is relatively slow, we aim to minimize the number of such transfers. We therefore maximize the mini-batch size subject to the available memory budget, denoted by $M$.

A common strategy for forming mini-batches from large graphs is to first partition the graph into clusters~\cite{Chiang_2019}, e.g., using the Leiden algorithm~\citep{traag2019leiden}, which scales to billion-edge graphs and optimizes modularity efficiently. However, cluster sizes can vary substantially, and the node/edge type composition within a cluster can deviate markedly from the global data distribution. 
In pre-training, we typically optimize on only a subset of the full graph (e.g., 1B nodes sampled from a 100B-node graph). In this regime, randomly selecting clusters to load into memory can (i) underutilize the desired batch size $B$ due to cluster-size variance, induce distributional bias when early updates are dominated by a few atypical clusters and lead to unstable training \citep{zeng2020graphsaintgraphsamplingbased, 10.1145/1553374.1553380}. This effect is amplified in multi-machine training, where the effective batch size per optimization step can reach tens to hundreds of millions of samples \citep{keskar2017largebatchtrainingdeeplearning}. Consequently, biased node/edge-type exposure in the first steps can steer optimization toward suboptimal solutions. To mitigate these issues, we propose \textit{KL-Batching}, a simple and efficient procedure that  assembles memory-efficient mini-batches whose type distributions are close to the global distribution.

KL-Batching first partitions the graph into $K$ disjoint clusters,
\(
\{C_1,\dots,C_K\}, \qquad C_k \subseteq V,
\)
which are never split across batches to avoid introducing additional cross-batch edge cuts.
For each cluster $C_k$, we compute an empirical distribution over a chosen discrete attribute of interest, denoted by $a(\cdot)$. This attribute can correspond to node types, edge types, or any other crucial categorical property used to control representativeness. Let $T$ denote the support of this attribute and let $p_k(t)$ be the empirical distribution induced by $C_k$ over $t \in T$. The full graph (or the pre-training population) similarly induces a global reference distribution $p_G(t)$.

For each cluster, we then compute the Kullback--Leibler (KL) divergence to the global distribution,
\[
\mathrm{KL}(p_k \,\|\, p_G)
=
\sum_{t \in T} p_k(t)\,\log\frac{p_k(t)}{p_G(t)},
\]
which quantifies how representative $C_k$ is with respect to the selected attribute distribution. When multiple attributes are important, one can compute multiple KL terms and combine them via, e.g., a weighted sum.

Given a target batch capacity $B$, measured as an upper bound on the total storage load associated with the batch, we construct batches by joining entire clusters. Let $\mathrm{size}(C_k)>0$ be a cost estimate for cluster $C_k$ (e.g., accounting for its number of nodes and edges, the mix of node/edge types, and type-specific feature dimensionalities). KL-Batching proceeds by (1) sorting all clusters in ascending order of $\mathrm{KL}(p_k \,\|\, p_G)$, and then (2) traversing this list and sequentially aggregating clusters into batches as long as the cumulative batch cost does not exceed the memory budget $M$.

This yields a collection of batches $\mathcal{B}=\{B_1,\dots,B_S\}$, where each batch $B_\ell$ is the disjoint union of whole clusters,
\[
B_\ell = \bigcup_{k \in I_\ell} C_k,
\qquad
\sum_{k \in I_\ell} \mathrm{size}(C_k) \le M,
\]
and $\mathrm{size}(\cdot)$ denotes the estimated memory cost. Because batches are assembled from low-KL clusters first and filled as close as possible to the capacity constraint, we (i) obtain batches that are representative with respect to the chosen attribute(s) and (ii) improve memory utilization at each training worker.
The batch construction stage can be viewed as a constrained combinatorial optimization problem; due to space limitations, we provide a formal formulation in the Appendix.

Suppose we have clusters indexed by $k=1,\dots,K$, each with a size (cost) $s_k = \mathrm{size}(C_k) > 0$ and a KL value
\[
\kappa_k = \mathrm{KL}(p_k \,\|\, p_G).
\]
Fix a batch capacity $B$. For any batch, we seek a subset of indices $I \subseteq \{1,\dots,K\}$ such that
\begin{align}
\sum_{k \in I} s_k &\le B, \label{eq:knapsack-capacity} \\
\sum_{k \in I} s_k &\text{ is maximized}, \label{eq:knapsack-max}
\end{align}
under the preference that clusters with lower $\kappa_k$ are chosen first (i.e., we want the batch to be composed of clusters whose node-type distribution is close to the global one).

If we fix a set of candidate clusters that all have the same KL value, $\kappa_k = \kappa^\star$ for all $k$ in some index set $\mathcal{K}^\star$, then selecting a subset $I \subseteq \mathcal{K}^\star$ maximizing~\eqref{eq:knapsack-max} subject to~\eqref{eq:knapsack-capacity} is exactly the classical $0$--$1$ knapsack problem. This problem is known to be NP-hard in general, and solving it exactly for every batch is computationally infeasible at the scales we consider.

In our setting, exact ties in $\kappa_k$ are rare in practice, because KL values are continuous and clusters exhibit diverse type distributions. Consequently, the number of genuine ``knapsack'' situations, where we must choose among many clusters with effectively identical KL values-, is small relative to the total number of batches. We therefore adopt a simple greedy heuristic:
\begin{itemize}
    \item We traverse clusters in ascending order of $\kappa_k$.
    \item For each batch, we add clusters sequentially as long as the capacity constraint $\sum s_k \le B$ is satisfied.
    \item When multiple candidate clusters share the same $\kappa_k$, we add them in arbitrary (or size-sorted) order, without attempting to solve the knapsack problem exactly.
\end{itemize}
This heuristic does not guarantee globally optimal packing with respect to $B$, but it works well empirically: capacity utilization is typically high (batches are close to full), and pre-training remains stable. Given the rarity of large same-KL groups and the overwhelming scale of the graph, the suboptimality introduced by this greedy step is negligible in practice.

\section{Additional Evaluation Information}

\subsection{Hyper-Parameters}
All task evaluations uses a grid search with layers in $\{1,2\}$, hidden dimensions in $\{64,128,256\}$, dropout ratio in $\{0, 0.2, 0.6\}$, learning rate in $\{0.001, 0.0001, 0.0005\}$. We use 1000 epochs with early stopping on the validation loss, with patience of $200$ steps.

\subsection{Enterprise Feature Distributions}
\label{appendix:enterprise_features}

To characterize the discriminative power of node-level features across diverse graph tasks, we visualize their distributional properties using two complementary families of plots.

\Cref{fig:n-double-features} provides a fine-grained view of individual feature distributions for Task 1. For each node-level feature, we show four complementary plots: Kernel Density Estimation revealing shape, modality, and spread; violin plots of node degree stratified by binned feature values illustrating how feature magnitude relates to graph connectivity; hexbin scatter plots of log-transformed feature values versus log node degree highlighting non-linear relationships; and histograms with KDE overlays. Features are ranked by absolute skewness.

\Cref{fig:n-raincloud} offers a global summary across many features by presenting raincloud plots for the top 120 features ranked by the Kolmogorov--Smirnov statistic. Each subfigure combines a half-violin, a box plot, and jittered scatter points, jointly conveying central tendency, variability, and distributional shape across tasks during pre-training.

\begin{figure*}[h]
  \centering
  \begin{subfigure}[t]{0.49\textwidth}
    \centering
    \includegraphics[width=\textwidth]{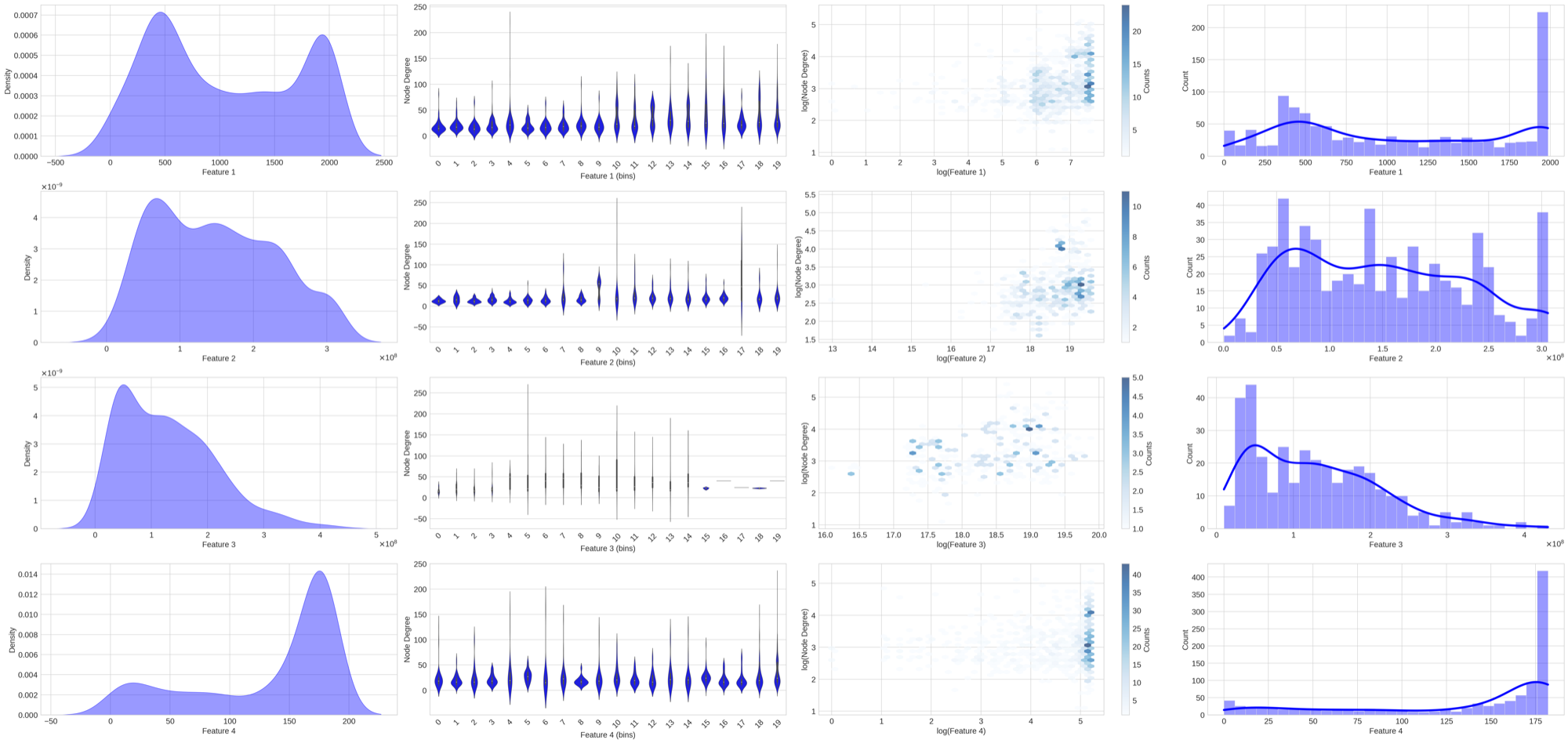}
    \caption{Feature group A.}
    \label{fig:n-double-features-a}
  \end{subfigure}
  \hfill
  \begin{subfigure}[t]{0.49\textwidth}
    \centering
    \includegraphics[width=\textwidth]{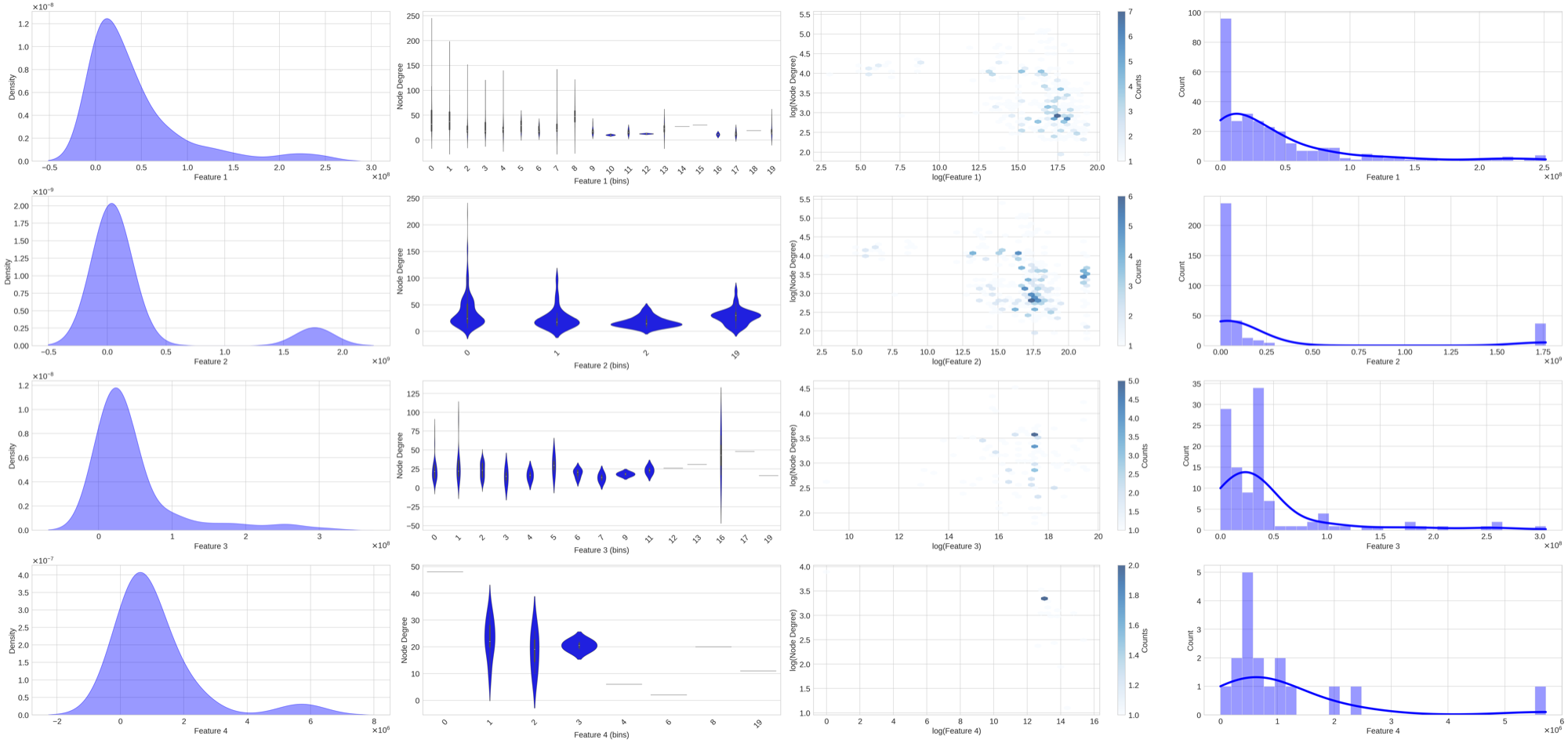}
    \caption{Feature group B.}
    \label{fig:n-double-features-b}
  \end{subfigure}

  \vspace{0.5em}

  \begin{subfigure}[t]{0.49\textwidth}
    \centering
    \includegraphics[width=\textwidth]{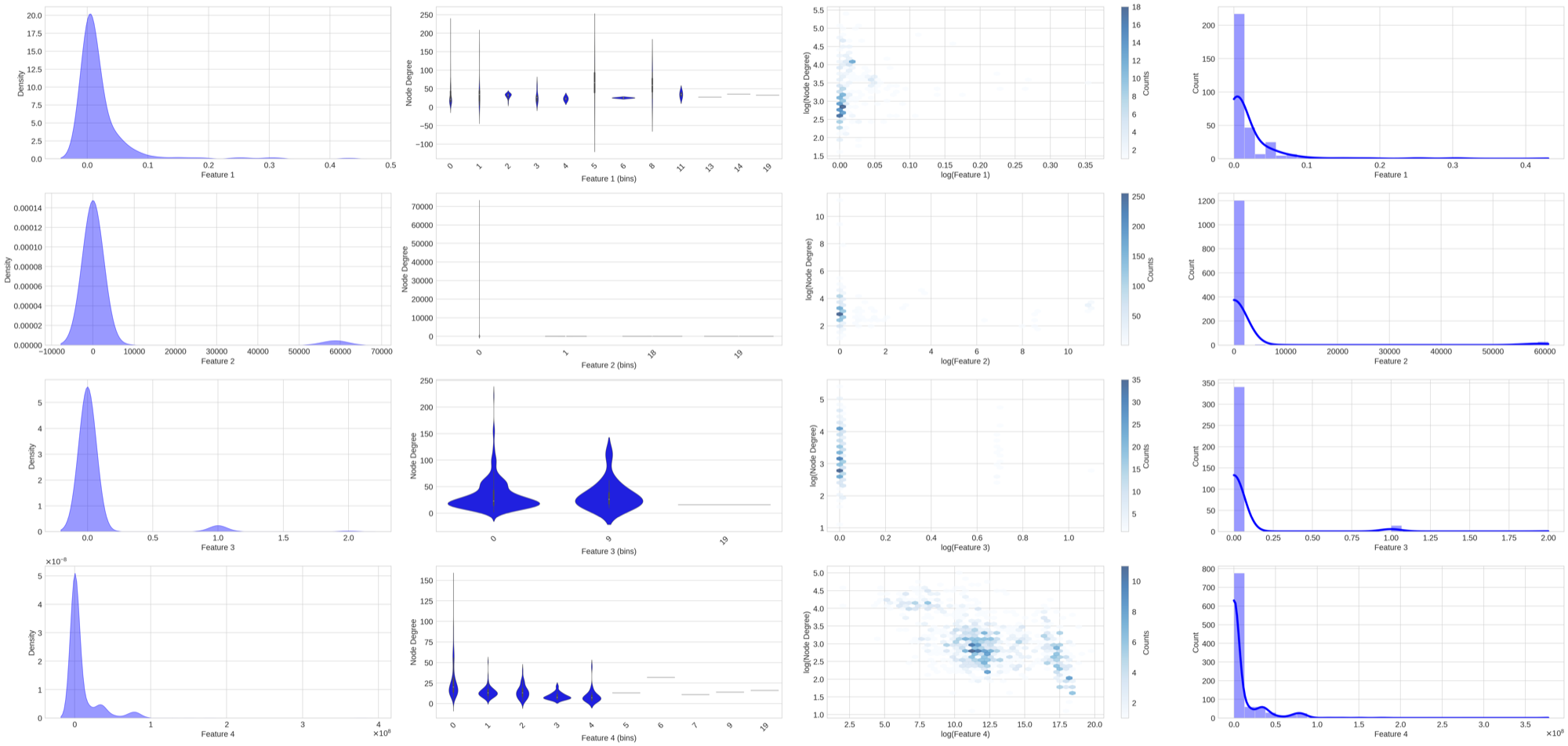}
    \caption{Feature group C.}
    \label{fig:n-double-features-c}
  \end{subfigure}
  \hfill
  \begin{subfigure}[t]{0.49\textwidth}
    \centering
    \includegraphics[width=\textwidth]{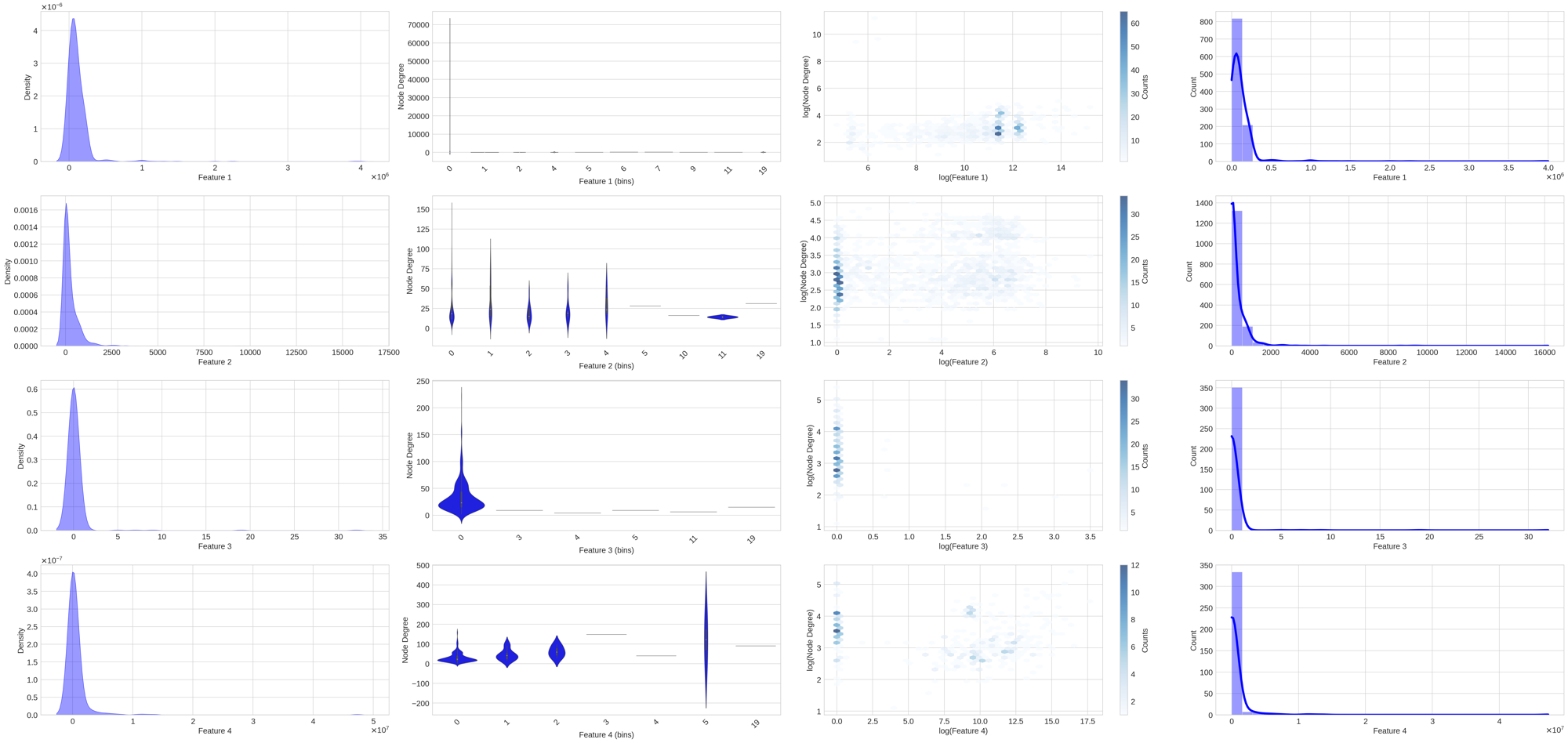}
    \caption{Feature group D.}
    \label{fig:n-double-features-d}
  \end{subfigure}

  \caption{Task 1: feature distribution comparisons over the pre-training data. Each panel shows KDE, degree-stratified violin plots, hexbin scatter plots, and histograms for a group of node-level features.}
  \label{fig:n-double-features}
\end{figure*}

\begin{figure*}[h]
  \centering
  \begin{subfigure}[t]{0.49\textwidth}
    \centering
    \includegraphics[width=\textwidth]{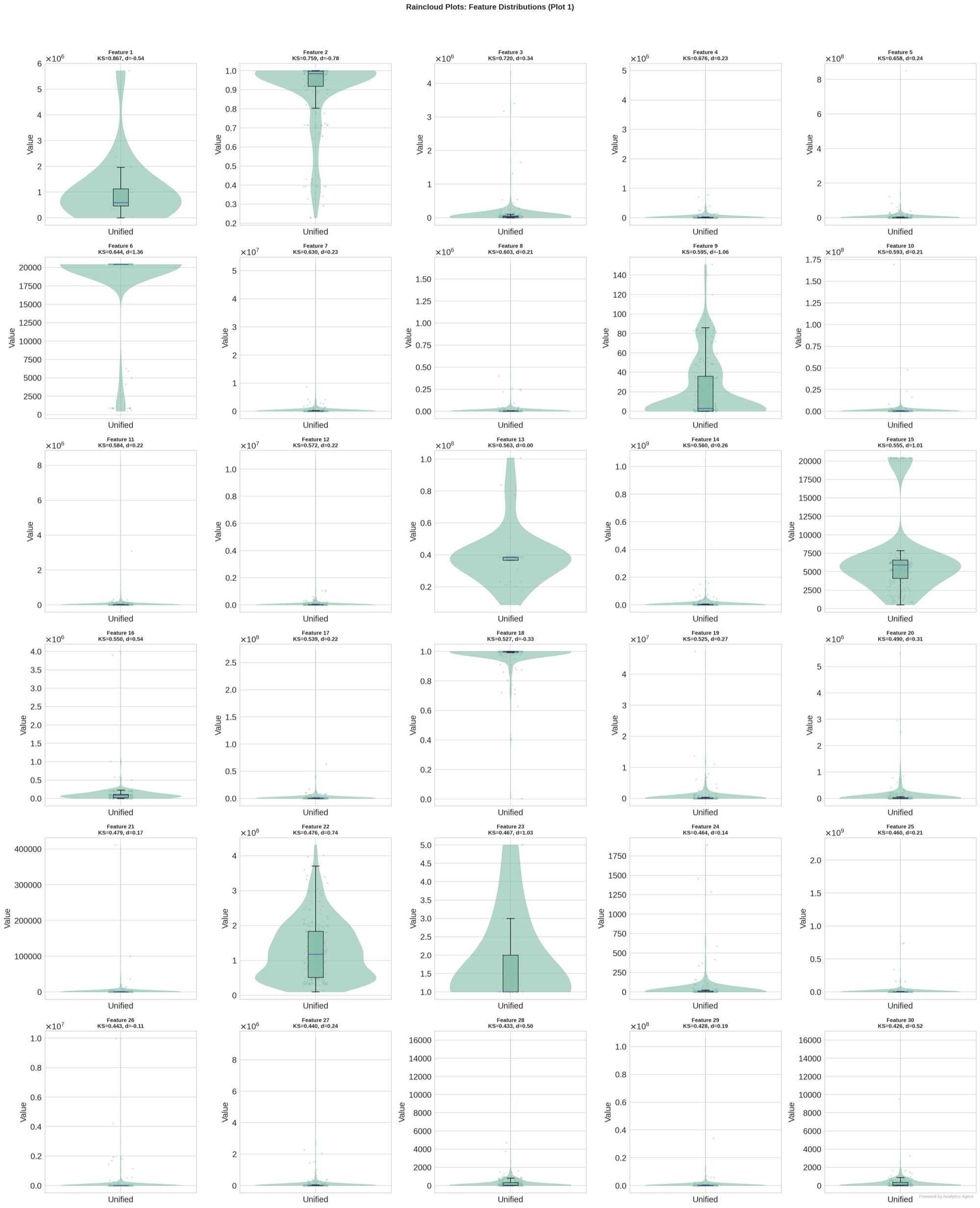}
    \caption{Raincloud for feature group A.}
    \label{fig:n-raincloud-a}
  \end{subfigure}
  \hfill
  \begin{subfigure}[t]{0.49\textwidth}
    \centering
    \includegraphics[width=\textwidth]{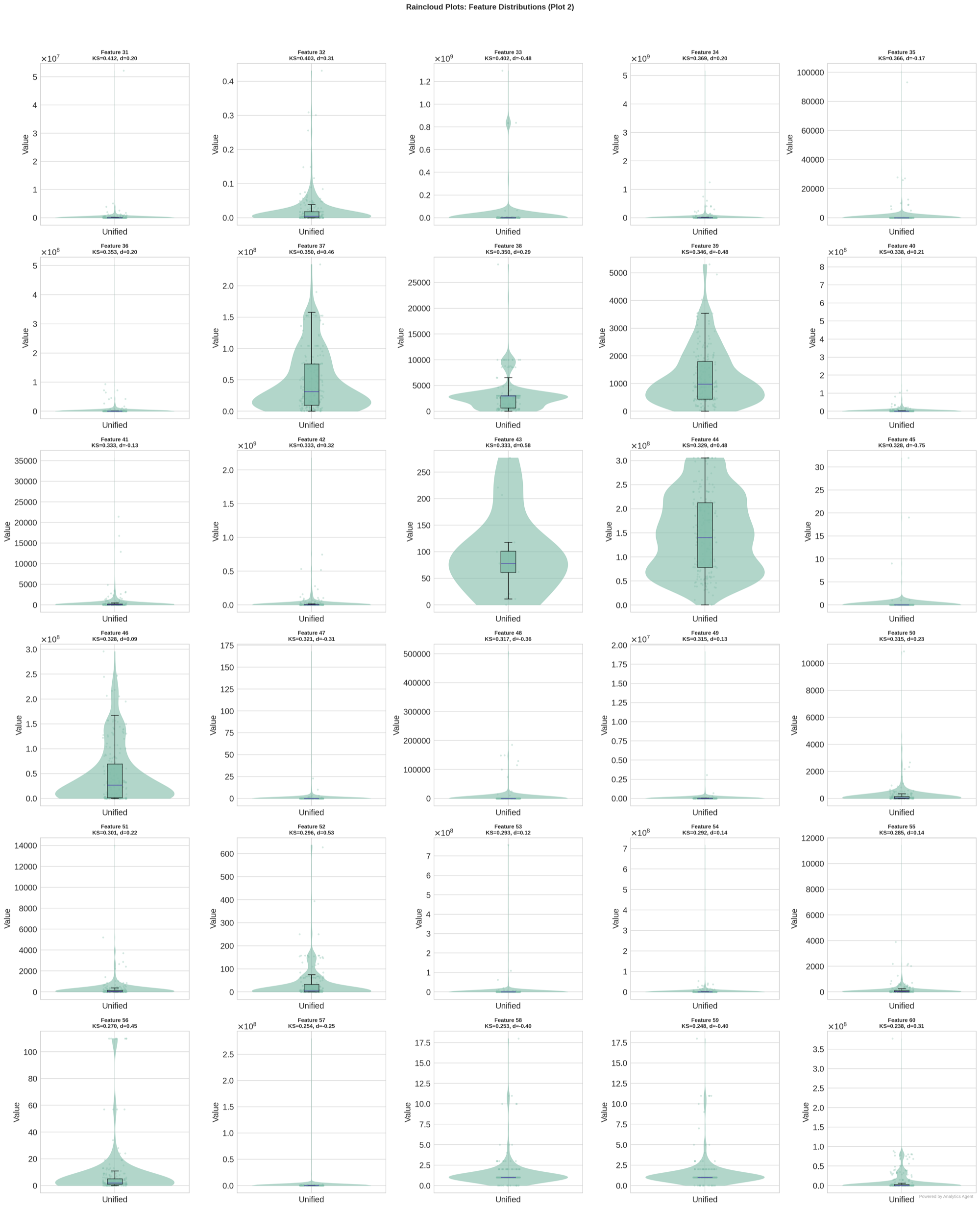}
    \caption{Raincloud for feature group B.}
    \label{fig:n-raincloud-b}
  \end{subfigure}

  \vspace{0.5em}

  \begin{subfigure}[t]{0.49\textwidth}
    \centering
    \includegraphics[width=\textwidth]{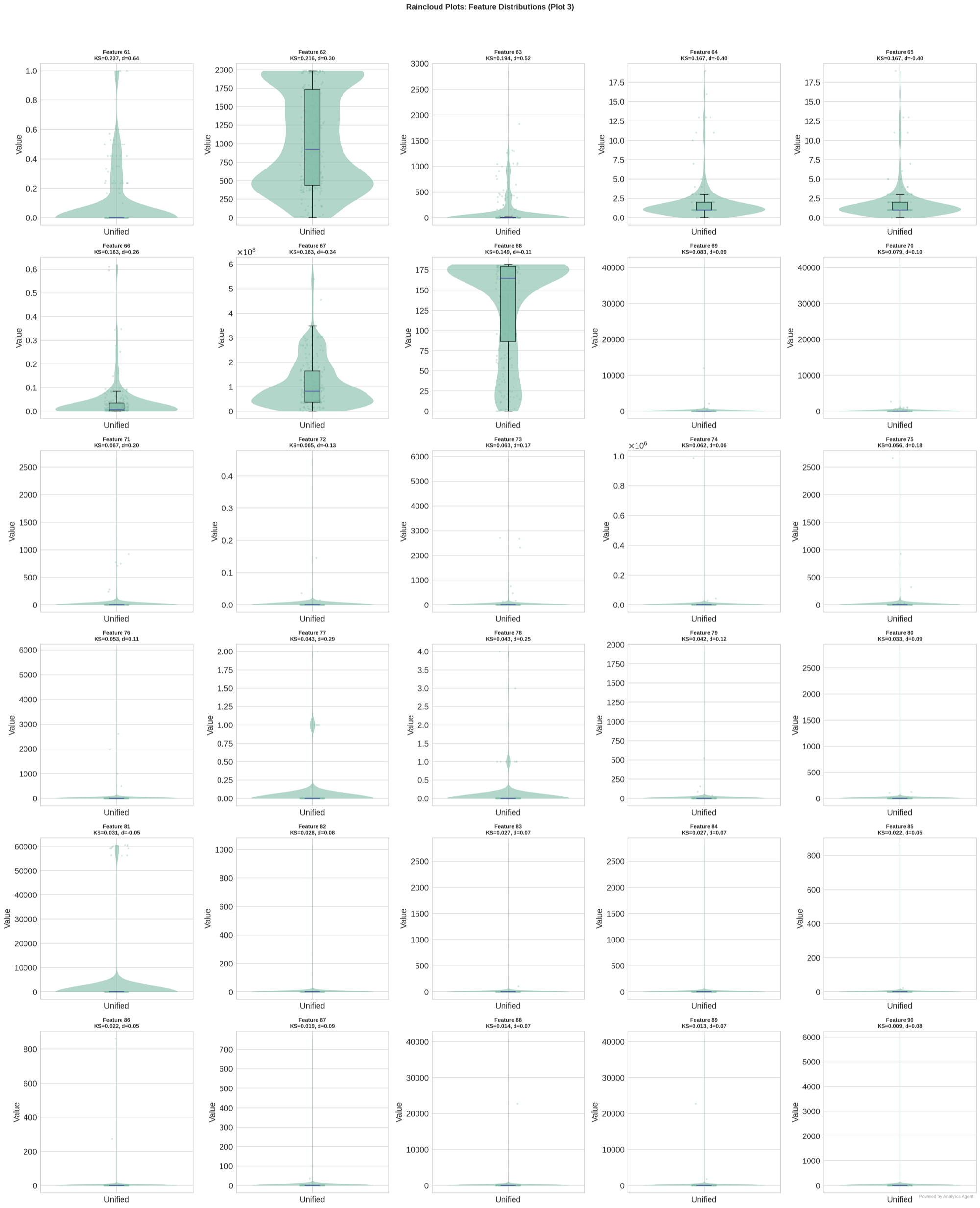}
    \caption{Raincloud for feature group C.}
    \label{fig:n-raincloud-c}
  \end{subfigure}
  \hfill
  \begin{subfigure}[t]{0.49\textwidth}
    \centering
    \includegraphics[width=\textwidth]{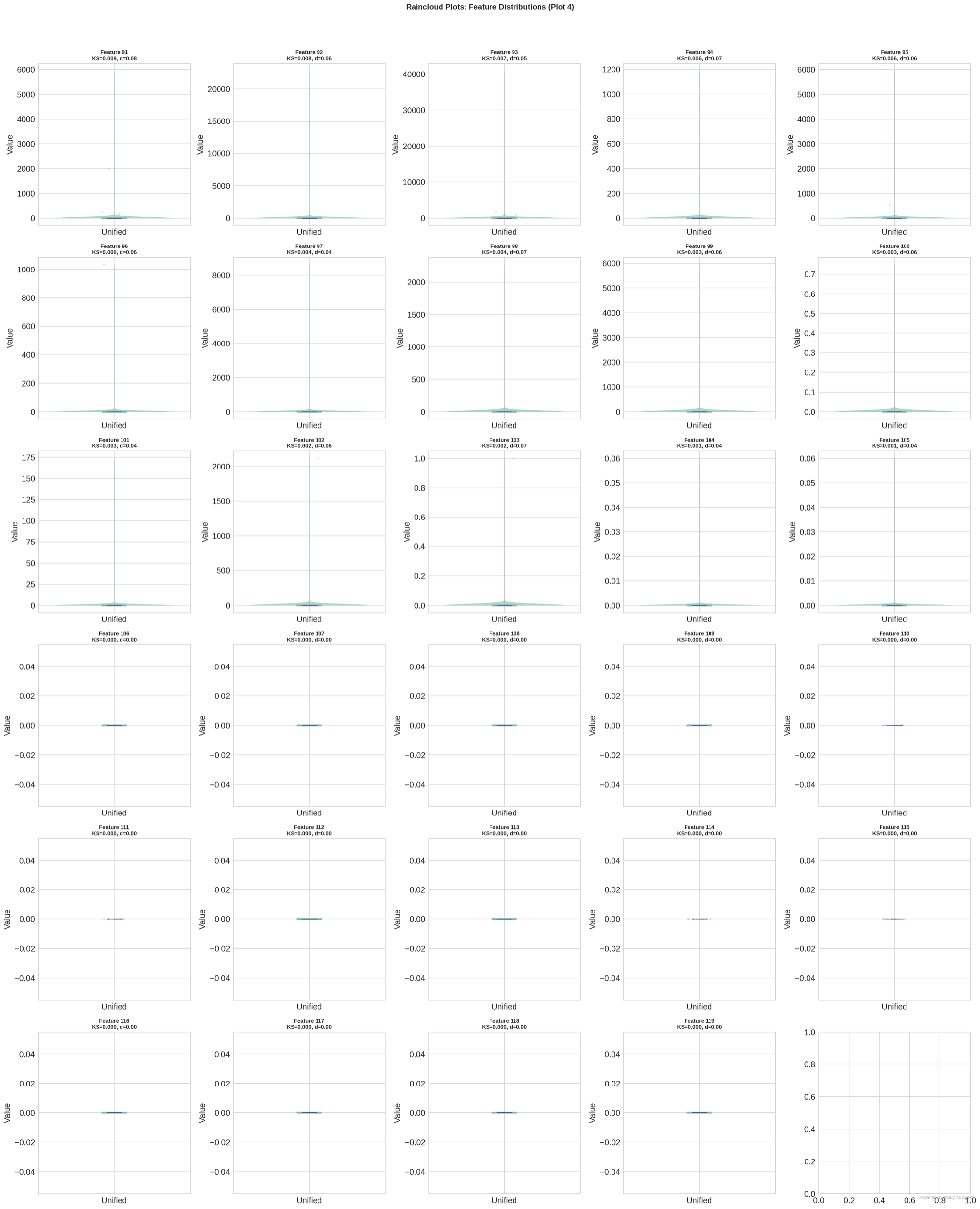}
    \caption{Raincloud for feature group D.}
    \label{fig:n-raincloud-d}
  \end{subfigure}

  \caption{Task 1: raincloud plots of feature distributions. Each panel combines a half-violin, box plot, and jittered scatter points for the top features ranked by the Kolmogorov--Smirnov statistic.}
  \label{fig:n-raincloud}
\end{figure*}

\section{Zero-Shot}

\begin{table}[t]
\caption{Zero-shot link prediction using the GFM’s pre-trained link predictor on three link-classification tasks. We evaluate three context sizes, corresponding to ego-graphs of varying radius centered on the target edge.}
 \label{tab:zero_shot}
\centering

\begin{tabular}{lccc}
\toprule
 \textbf{Model} & \textbf{Task 8} &  \textbf{Task 9} &  \textbf{Task 10}  \\
  \textbf{- Context} & \scriptsize PRAUC & \scriptsize PRAUC & \scriptsize PRAUC \\
 \midrule
GFM-1 & 54.23 &77.13 &73.05 \\
GFM-2 & 55.35 & 39.48 &53.21 \\
GFM-3 &  49.89&41.16 &73.84\\
\bottomrule
\end{tabular}

\end{table}

\textbf{Zero-Shot setup} 
For zero-shot, we evaluate the link-prediction tasks using the frozen link-prediction head of the pre-trained GFM. Node-level zero-shot is not well-defined for nodes with a feature dimension that is different than the task prediction dimension. The same issue arises in vision tasks, as the ViT output is in high dimension, and some projection to the label space must be learned. As recently noted by \cite{eremeev2025turningtabularfoundationmodels}, existing ``zero-shot'' node level evaluation uses labeled data from the pre-training graph, similarly to transduction settings, which essentially violates the definition of zero-shot prediction. Moreover, this evaluation is not possible if inference tasks data have no intersection with the pre-training data at all, as in our case. Therefore, for the node-level tasks we perform zero-shot separation analysis instead.
Our few-shot analysis reveals that GFM representations are highly expressive even with minimal supervision. While performance generally improves with more data, we observe that the increase from 1-shot to 10-shot is not always monotonic, echoing findings in other domains \citep{luo2023closerlookfewshotclassification, radford2021learningtransferablevisualmodels} where low-shot regimes can be sensitive to the specific samples selected. Nevertheless, the 10-shot GFM performance frequently rivals or exceeds the full-dataset performance of standard HGT and HAN models. A striking example is seen in Task 9, where the 10-shot GFM-1 (87.36) outperforms all versions of HGT and HAN trained on the full dataset, which fail to break the 40.00 PRAUC barrier, underscoring the efficiency and superior feature alignment of the learned foundation representations.

\Cref{fig:zero_separation} presents the zero-shot separation results for Tasks 1 to 6. With just the raw features, we do not observe meaningful separation between the positive and negative classes. In contrast, the GFM yields strong zero-shot separation.
The zero-shot link-prediction performance is presented in \Cref{tab:zero_shot}. For Task 10, GFM in the zero-shot setting performs strictly better than all GNN baselines, even though those baselines are trained on the full task dataset, with an absolute improvement of $10.08$ points.

\begin{figure*}[t!]
  \centering

  \setlength{\tabcolsep}{0.5pt}
  \renewcommand{\arraystretch}{0.80}

  \begin{tabular}{@{}c cccc@{}}
    & \multicolumn{1}{c}{\textbf{\scriptsize Raw}}
    & \multicolumn{1}{c}{\textbf{\scriptsize GFM-1}}
    & \multicolumn{1}{c}{\textbf{\scriptsize GFM-2}}
    & \multicolumn{1}{c}{\textbf{\scriptsize GFM-3}} \\[-2pt]

    \rotatebox{90}{\scriptsize Task 1} &
    \includegraphics[width=.235\columnwidth]{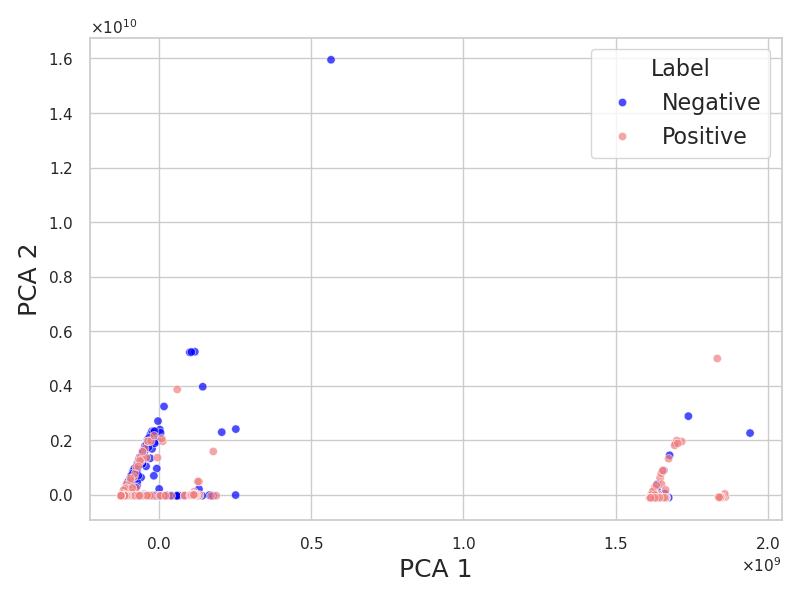} &
    \includegraphics[width=.235\columnwidth]{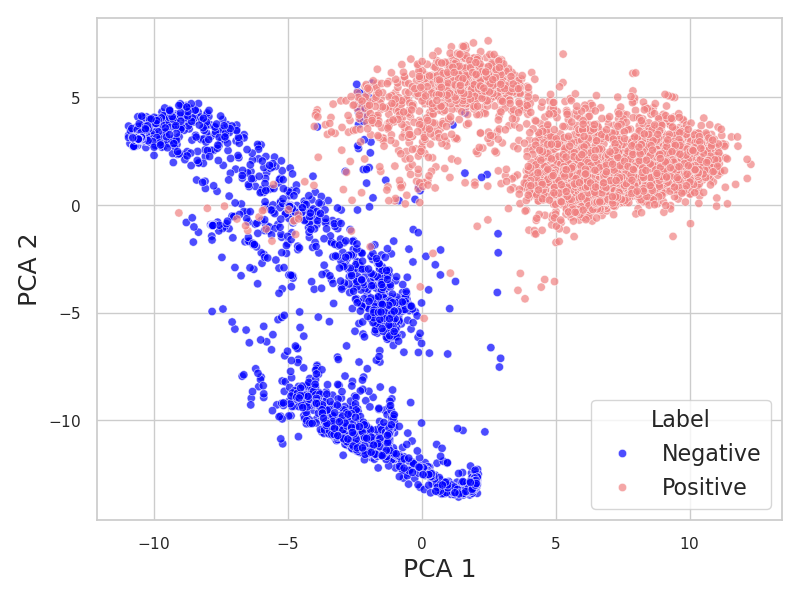} &
    \includegraphics[width=.235\columnwidth]{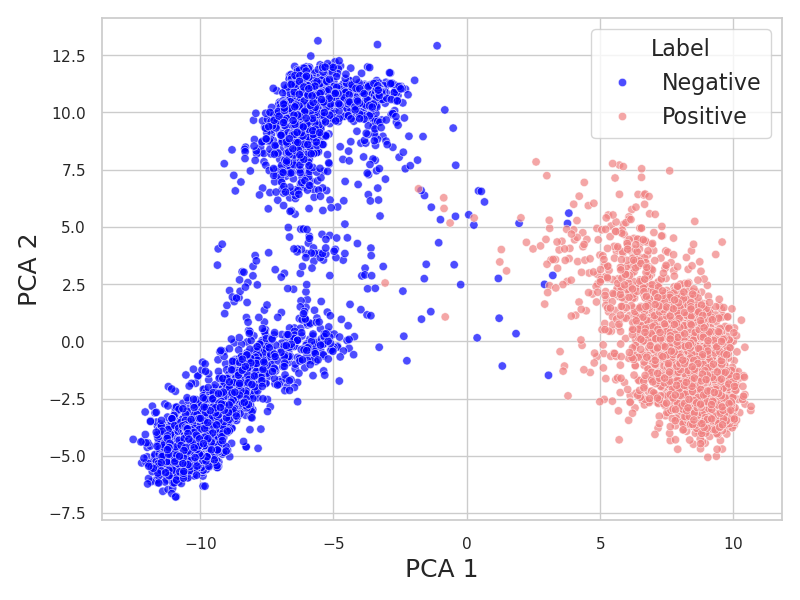} &
    \includegraphics[width=.235\columnwidth]{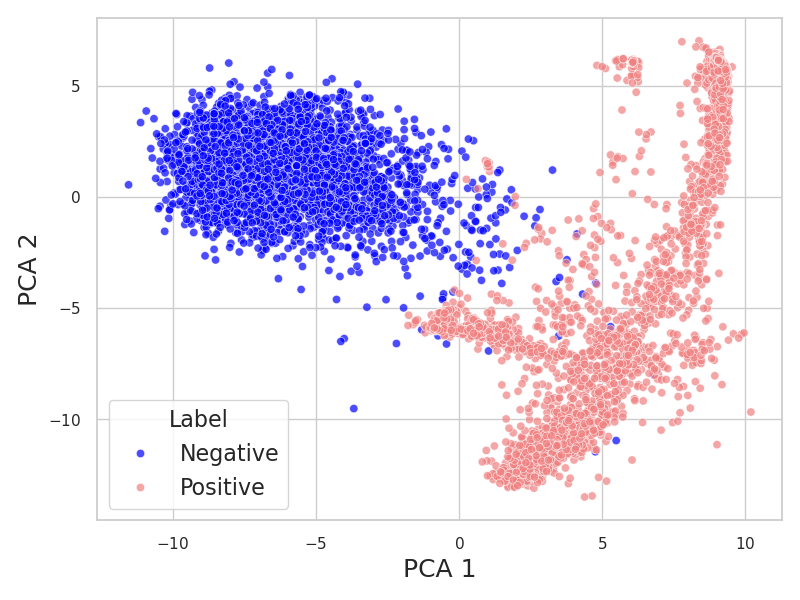} \\[-2pt]

    \rotatebox{90}{\scriptsize Task 2} &
    \includegraphics[width=.235\columnwidth]{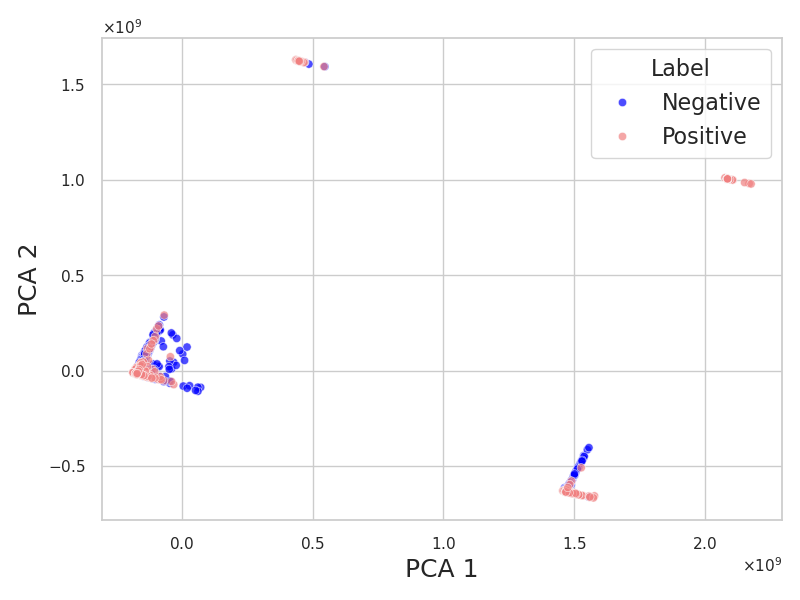} &
    \includegraphics[width=.235\columnwidth]{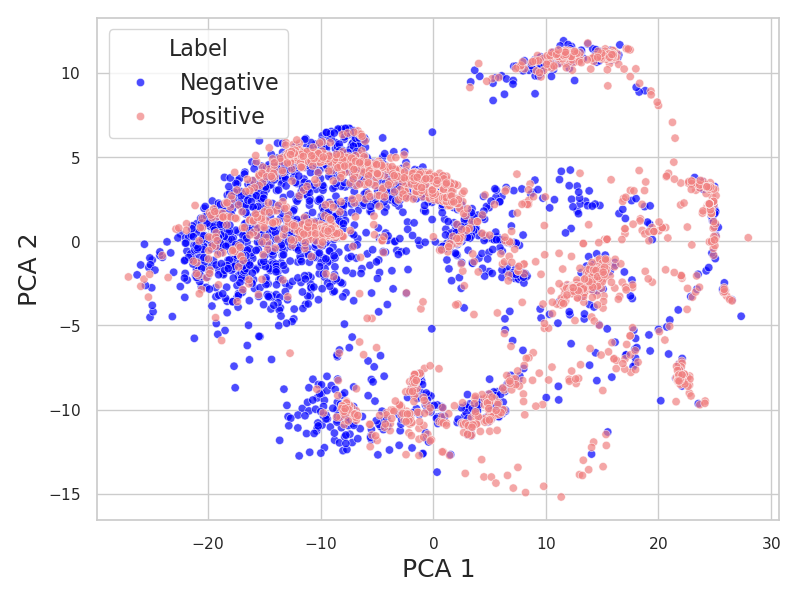} &
    \includegraphics[width=.235\columnwidth]{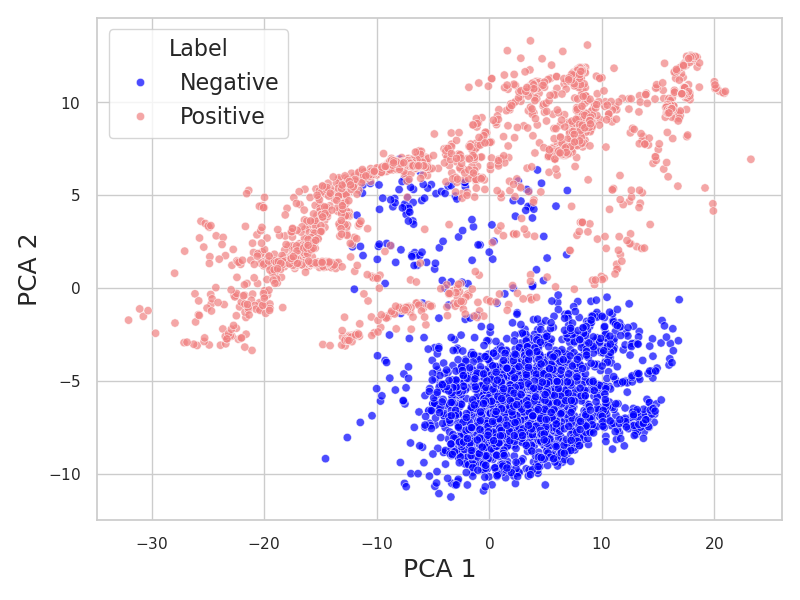} &
    \includegraphics[width=.235\columnwidth]{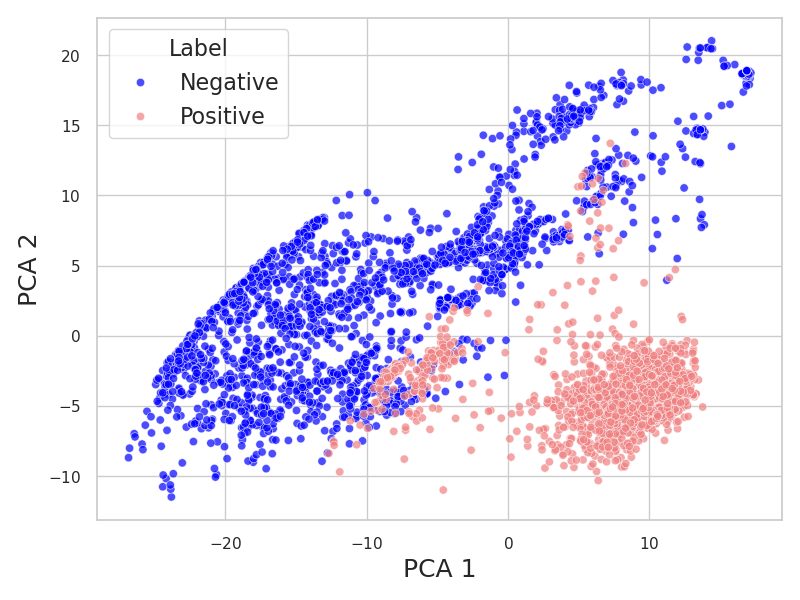} \\[-2pt]

    \rotatebox{90}{\scriptsize Task 3} &
    \includegraphics[width=.235\columnwidth]{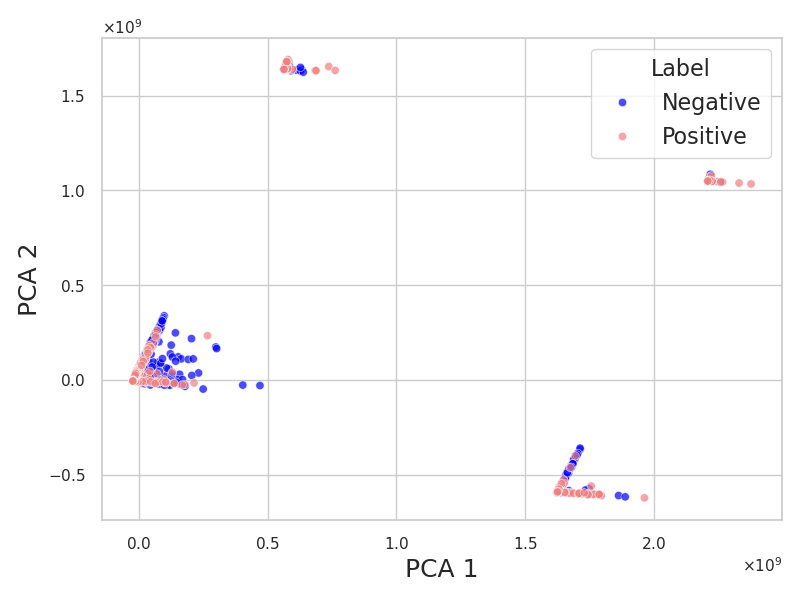} &
    \includegraphics[width=.235\columnwidth]{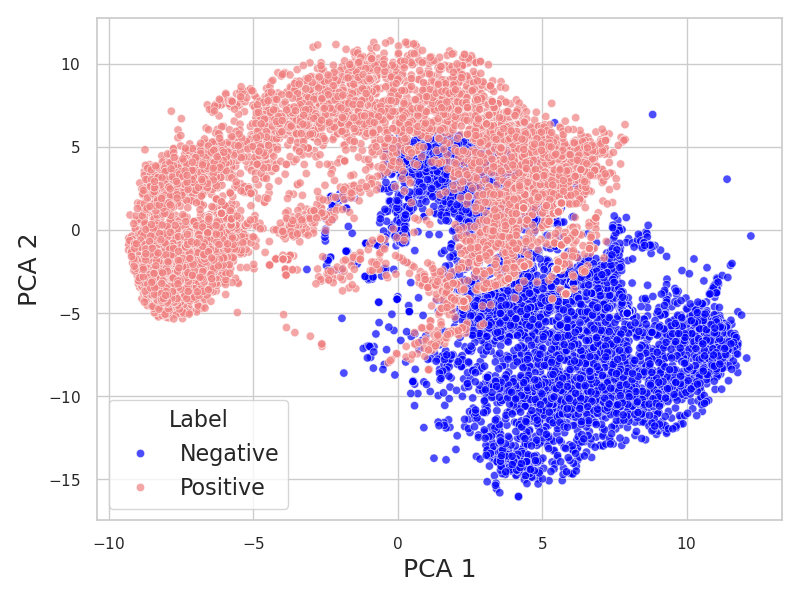} &
    \includegraphics[width=.235\columnwidth]{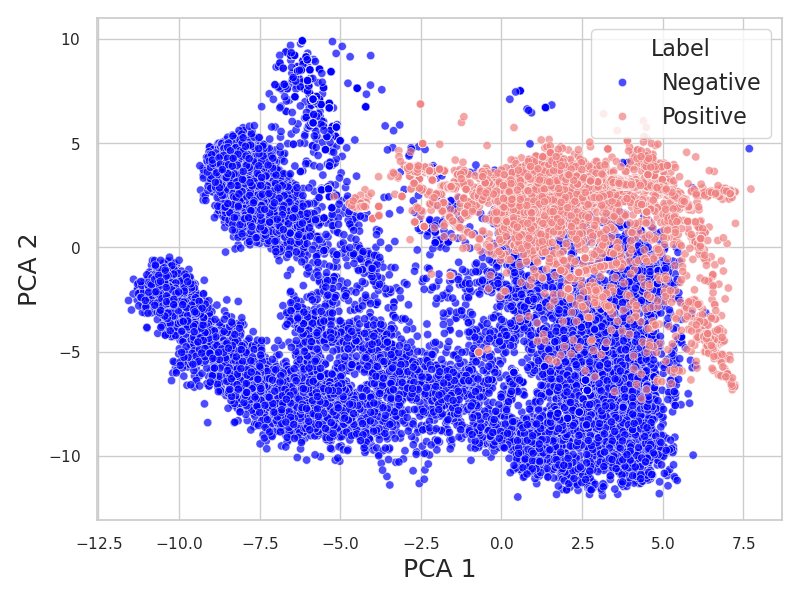} &
    \includegraphics[width=.235\columnwidth]{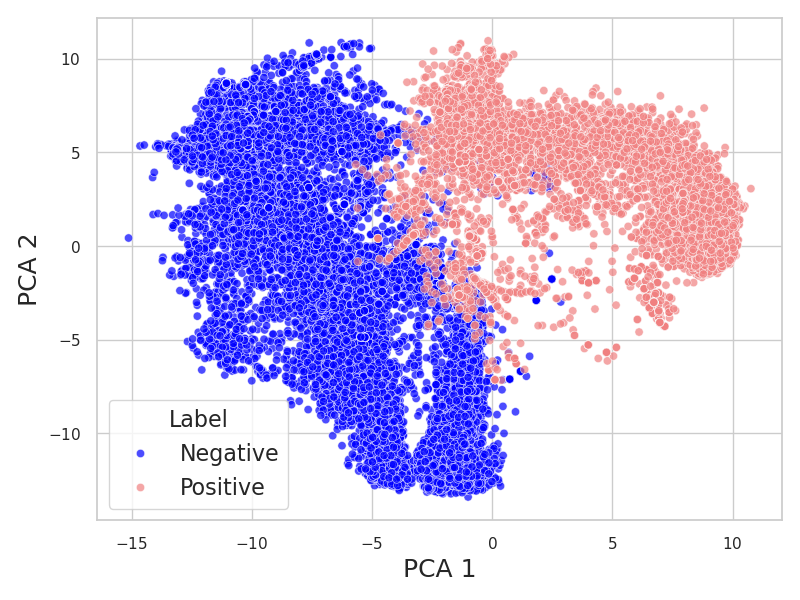} \\[-2pt]

    \rotatebox{90}{\scriptsize Task 4} &
    \includegraphics[width=.235\columnwidth]{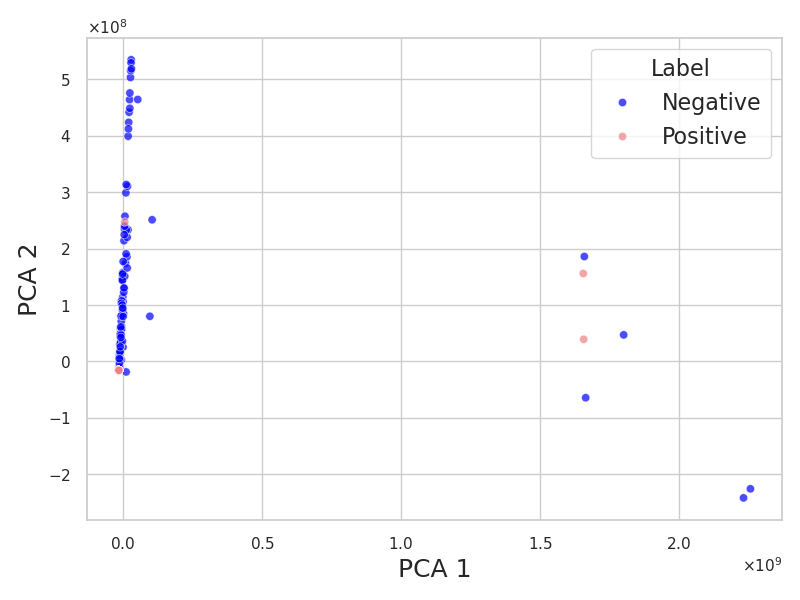} &
    \includegraphics[width=.235\columnwidth]{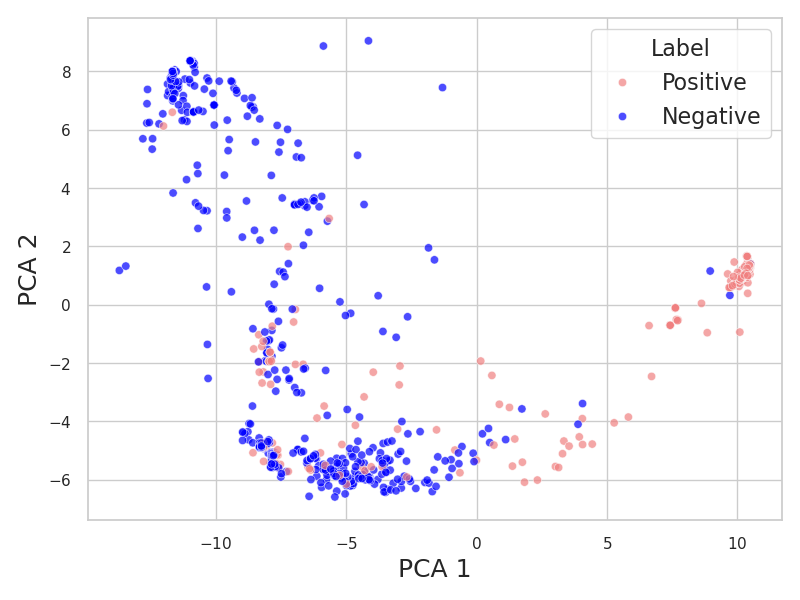} &
    \includegraphics[width=.235\columnwidth]{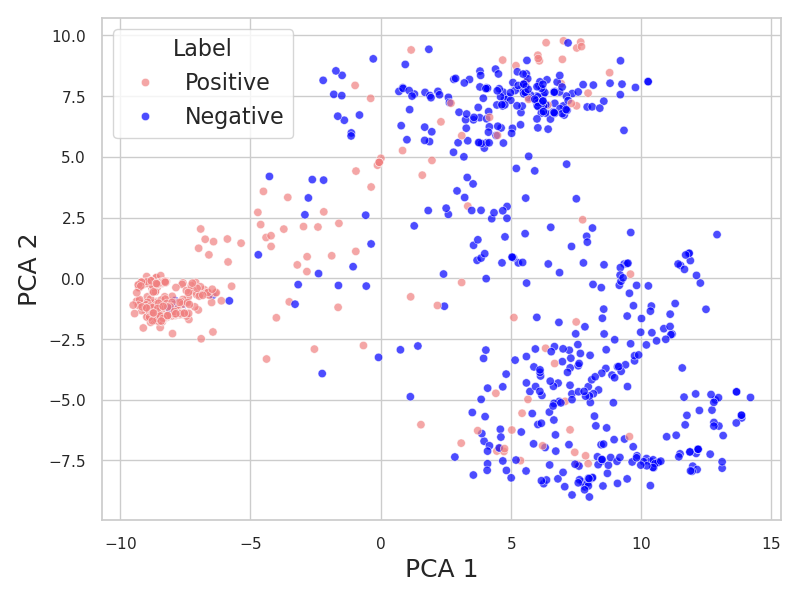} &
    \includegraphics[width=.235\columnwidth]{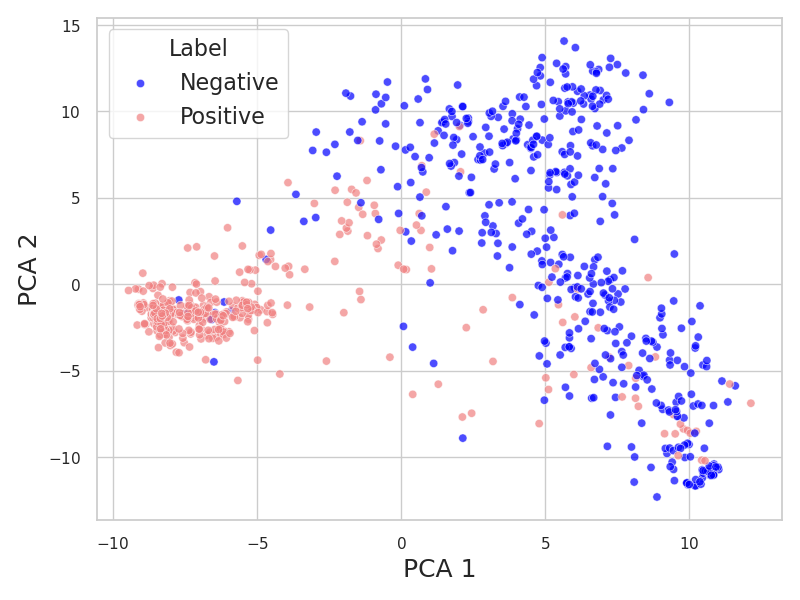} \\[-2pt]

    \rotatebox{90}{\scriptsize Task 5} &
    \includegraphics[width=.235\columnwidth]{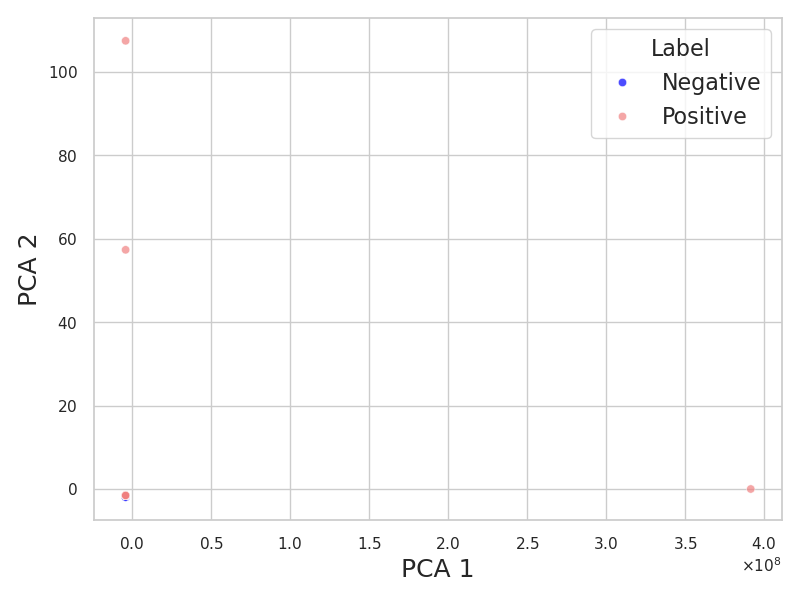} &
    \includegraphics[width=.235\columnwidth]{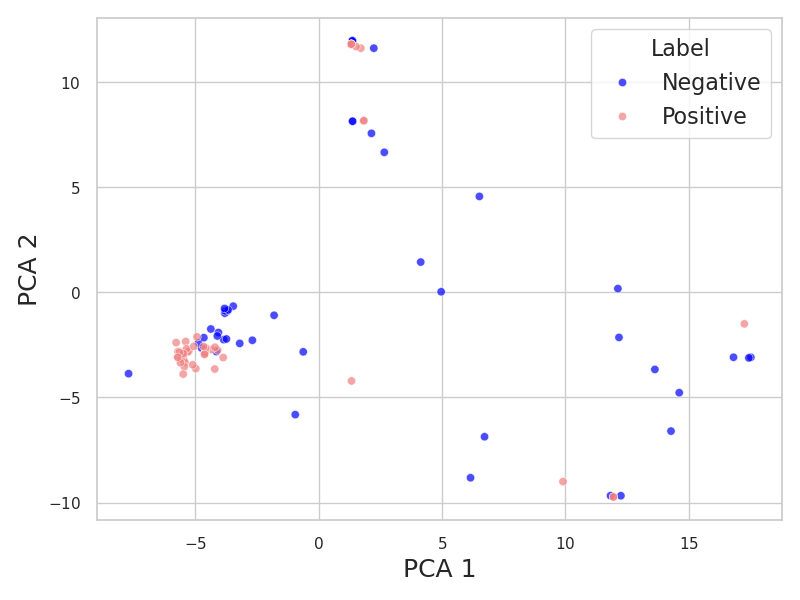} &
    \includegraphics[width=.235\columnwidth]{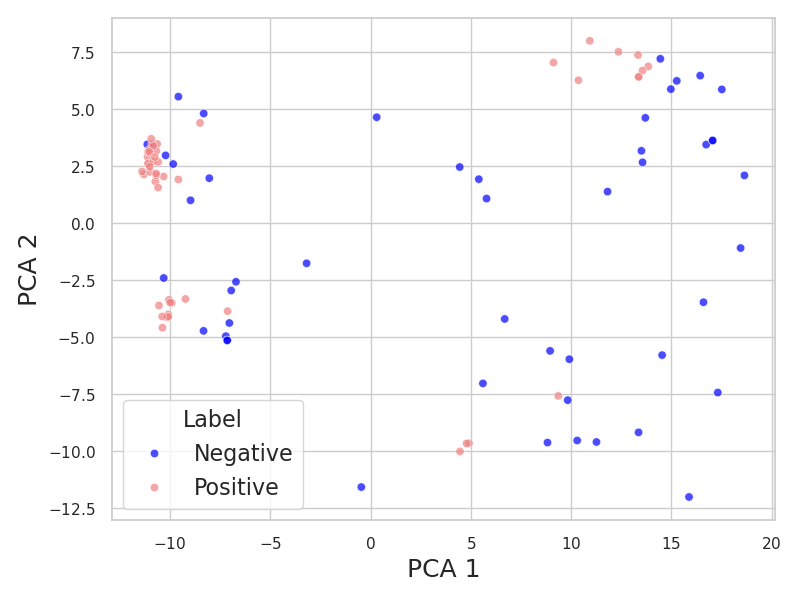} &
    \includegraphics[width=.235\columnwidth]{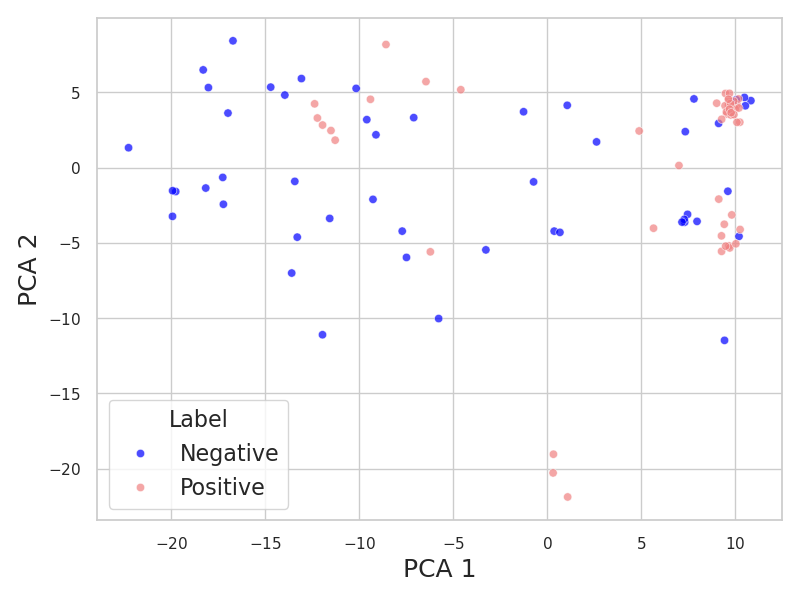} \\[-2pt]

    \rotatebox{90}{\scriptsize Task 6} &
    \includegraphics[width=.235\columnwidth]{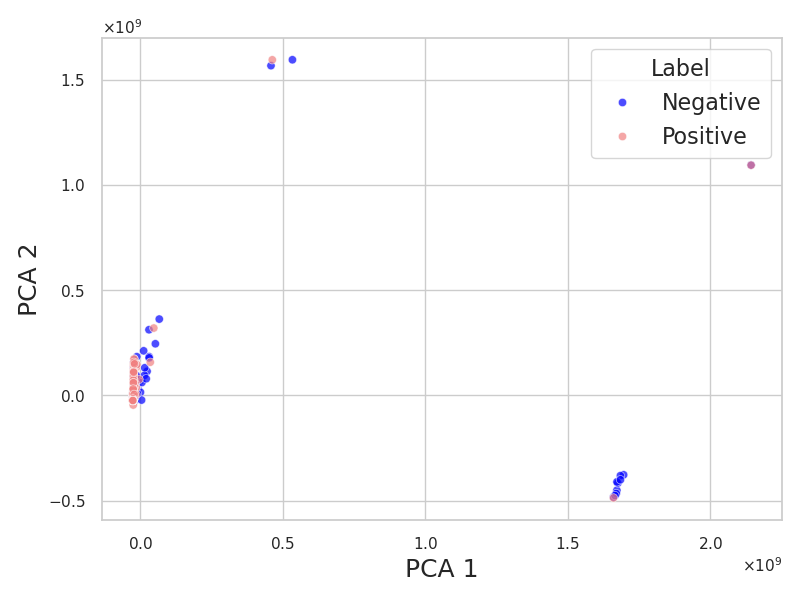} &
    \includegraphics[width=.235\columnwidth]{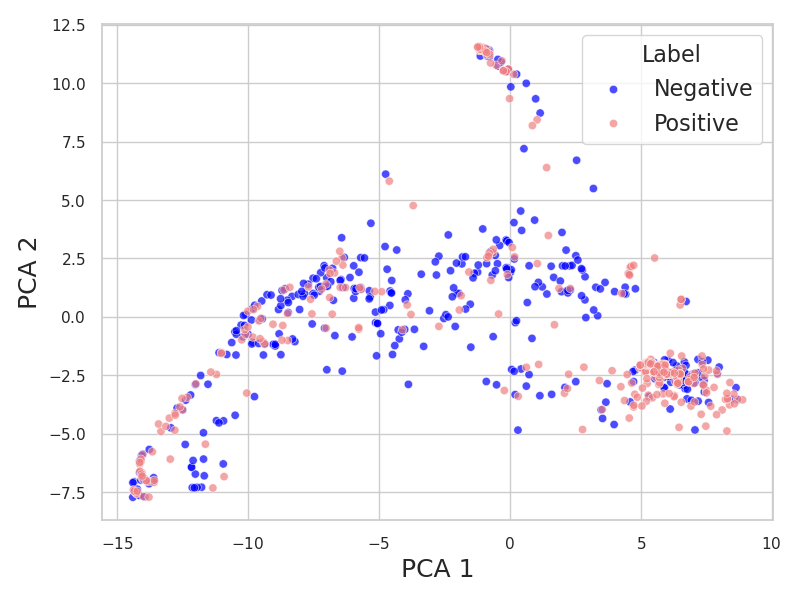} &
    \includegraphics[width=.235\columnwidth]{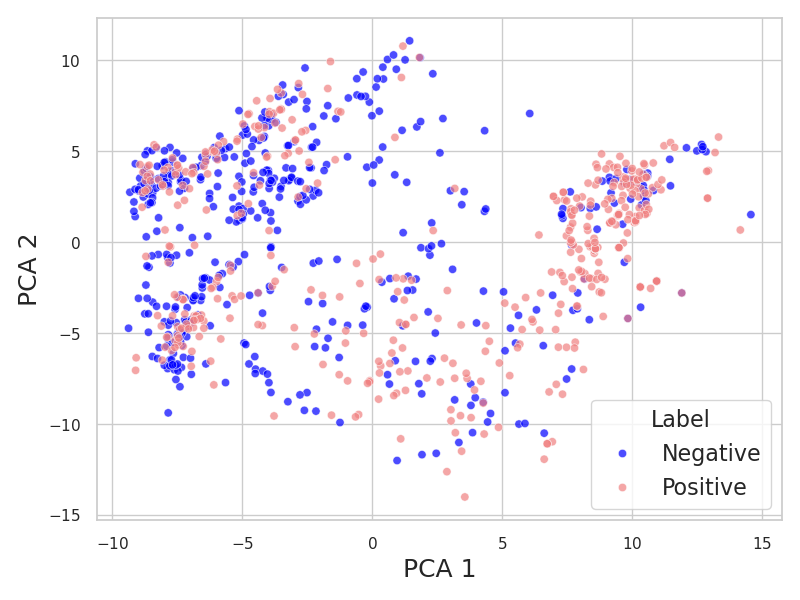} &
    \includegraphics[width=.235\columnwidth]{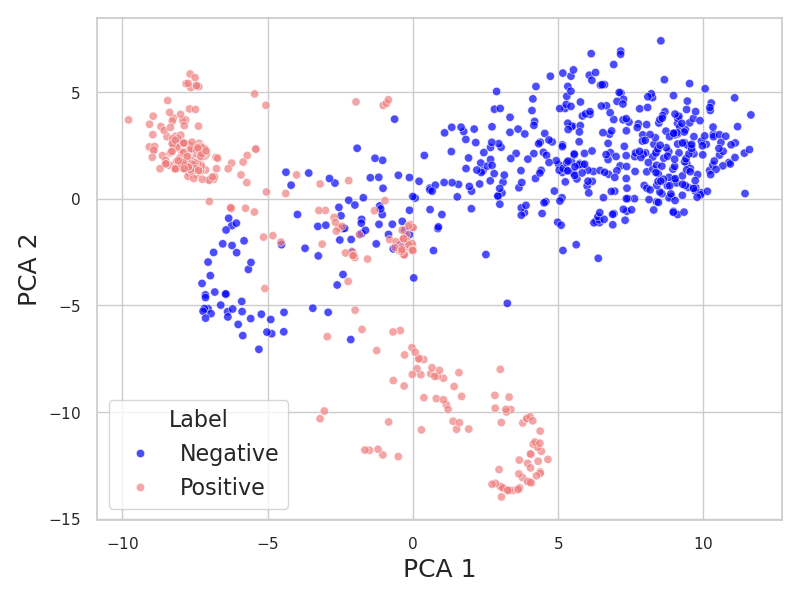} \\
  \end{tabular}
  \caption{Zero-shot separation for Enterprise Tasks 1--6. The raw features and GFM embeddings for context sizes 1, 2, and 3 are projected to 2D space using PCA, and the points are colored according to their label.}
    \label{fig:zero_separation}
\end{figure*}

\section{Time Complexity Analysis}
\label{app:time_complexity}

We compare the per-layer time and memory complexity of \method{} against HAN and HGT. Let $B$ denote the batch size, $d$ the hidden dimension, $h$ the number of attention heads, $d_h = d/h$ the head dimension, $|E_b|$ the number of sampled edges in the batch, and $q$ the fixed sampled degree in \method{}-TAA. For HAN, $K$ is the number of heads, $P$ the number of meta-paths, and $|E_\Phi|$ the number of node pairs for meta-path $\Phi$. For \method{}, $\mathcal{S}$ denotes the edge-type subsets used by TCA.

\textbf{HAN} requires $\mathcal{O}(K \sum_{\Phi=1}^{P} |E_\Phi| d_h + P B d)$ time per layer, as it repeats attention computation for each of $P$ meta-paths and $K$ heads, followed by semantic-level attention across meta-paths. Memory scales as $\mathcal{O}(K \sum_\Phi |E_\Phi| + P B d)$. The main bottleneck is meta-path expansion and materialization.

\textbf{HGT} requires $\mathcal{O}(|E_b| d)$ time, applying relation-aware heterogeneous multi-head attention with type-specific projections over all sampled edges. Memory scales as $\mathcal{O}(|E_b| h + B d)$. The bottleneck is the single softmax over all sampled neighbors of each target node, which becomes expensive for high-degree nodes.

\textbf{\method{}} requires $\mathcal{O}(|E_b^{\mathrm{tca}}| d_h + B q d_h)$ time when the edge-type subsets in $\mathcal{S}$ partition the sampled edges. The TCA component applies sparse softmax independently over each type-specific neighbor subset rather than a single softmax over the full neighborhood, reducing the per-node attention cost. The TAA component's cost is bounded by $Bq d_h$ regardless of actual node degree due to fixed-degree sampling. Memory scales as $\mathcal{O}(\sum_{v \in B} \sum_{S \in \mathcal{S}} |N_v^S| + Bq + Bd)$. Together, these choices enable \method{} to scale to billion-scale graphs with highly non-regular degree distributions.

\section{Representational Analysis}
\label{app:representational_analysis}

To understand the internal dynamics of the pretrained \method{} Transformer, we analyze how representations evolve across its 5-layer hierarchy on diverse unseen downstream tasks.

\paragraph{Dirichlet Energy and Attention Entropy.}
\Cref{tab:hierarchical_all_tasks} tracks the normalized Dirichlet Energy of node embeddings and the mean Shannon Entropy of attention weights across edge types at each layer. Across all tasks, the Dirichlet Energy follows a non-monotonic pattern, typically peaking at layers 2--3 before decreasing, indicating that the architecture acts as an adaptive filter that first sharpens node-distinct signals before aggregation in later layers. Simultaneously, attention entropy consistently decreases across layers, demonstrating that the model progressively narrows its focus to task-relevant structural motifs.

\begin{table*}[h]
\centering
\small
\caption{Dirichlet Energy and Attention Entropy across layers for five tasks. The non-monotonic DE trend and consistent entropy collapse demonstrate adaptive filtering and progressive structural focus.}
\label{tab:hierarchical_all_tasks}
\begin{tabular}{lccccc}
\toprule
\textbf{Task / Metric} & \textbf{Layer 1} & \textbf{Layer 2} & \textbf{Layer 3} & \textbf{Layer 4} & \textbf{Layer 5} \\
\midrule
\rowcolor[HTML]{EFEFEF} \multicolumn{6}{l}{\textit{Task 5}} \\
Dirichlet Energy (Norm) & 706.35 & 805.93 & 795.73 & 800.90 & 578.40 \\
Attention Entropy ($H$) & 0.1426 & 0.1062 & 0.0807 & 0.0708 & 0.0723 \\
\midrule
\rowcolor[HTML]{EFEFEF} \multicolumn{6}{l}{\textit{Task 8}} \\
Dirichlet Energy (Norm) & 703.71 & 799.15 & 786.38 & 802.19 & 565.15 \\
Attention Entropy ($H$) & 0.1498 & 0.1251 & 0.0961 & 0.0853 & 0.0887 \\
\midrule
\rowcolor[HTML]{EFEFEF} \multicolumn{6}{l}{\textit{Task 4}} \\
Dirichlet Energy (Norm) & 690.80 & 809.73 & 789.88 & 723.33 & 597.99 \\
Attention Entropy ($H$) & 0.1933 & 0.1448 & 0.1135 & 0.1024 & 0.1086 \\
\midrule
\rowcolor[HTML]{EFEFEF} \multicolumn{6}{l}{\textit{Task 9 (Source-centric)}} \\
Dirichlet Energy (Norm) & 263.91 & 474.10 & 487.53 & 385.13 & 480.67 \\
Attention Entropy ($H$) & 0.1605 & 0.1267 & 0.1094 & 0.0875 & 0.0862 \\
\midrule
\rowcolor[HTML]{EFEFEF} \multicolumn{6}{l}{\textit{Task 9 (Destination-centric)}} \\
Dirichlet Energy (Norm) & 267.50 & 444.52 & 452.95 & 359.00 & 430.77 \\
Attention Entropy ($H$) & 0.1859 & 0.1586 & 0.1399 & 0.1101 & 0.1075 \\
\bottomrule
\end{tabular}
\end{table*}

\paragraph{In-Context Structural Interrogation.}
To assess whether the model relies on node identity shortcuts or genuinely leverages neighborhood structure, we compare embeddings computed with original node features against embeddings where features are replaced by a mask token (\Cref{tab:mask_sensitivity_expanded}). The moderate cosine similarity between masked and unmasked embeddings, ranging from 0.31 to 0.76, confirms that the model performs a significant representational shift during masked inference. This shift indicates active interrogation of the neighborhood topology to reconstruct missing information, which is the mechanism underlying the model's zero-shot transfer capabilities.

\begin{table*}[h]
\centering
\caption{Cosine similarity and $L_2$ distance between unmasked and masked node embeddings across tasks. Moderate similarity confirms the model leverages neighborhood structure rather than relying on identity shortcuts.}
\label{tab:mask_sensitivity_expanded}
\begin{tabular}{llcc}
\toprule
\textbf{Context} & \textbf{Node Type} & \textbf{Cosine Similarity} & \textbf{$L_2$ Distance} \\
\midrule
Task 5 & NodeTypeC & $0.5784 \pm 0.2819$ & $12.698 \pm 4.332$ \\
Task 8 & NodeTypeC & $0.3111 \pm 0.1425$ & $17.858 \pm 1.806$ \\
Task 4 & NodeTypeC & $0.4142 \pm 0.2479$ & $15.418 \pm 3.627$ \\
Task 9 (Src) & NodeTypeA & $0.7614 \pm 0.1833$ & $9.135 \pm 3.100$ \\
Task 9 (Dst) & NodeTypeA & $0.7546 \pm 0.2084$ & $9.258 \pm 3.434$ \\
\bottomrule
\end{tabular}
\end{table*}

\section{Scaling Laws: Additional Details}
\label{app:scaling_laws_details}

\paragraph{Definition of Data Size.}
We define data size $D$ as the number of \emph{distinct target edges} used for supervision.
Repeated passes over the same edges are treated as optimization steps and do not increase $D$.
This definition follows classical learning-theoretic practice and is appropriate here
because the graph provides a virtually unbounded supply of unique edges.
Unlike language modeling, where data is often measured in tokens processed due to streaming
corpora and duplicate content~\citep{kaplan2020scaling,hoffmann2022training},
our setting allows an unambiguous notion of unique training examples.

\paragraph{Compute-Optimal Allocation.}
While scaling with respect to model size $N$ and data size $D$ is well-defined for arbitrary graphs, compute-optimal tradeoffs require special care. In graph models with neighborhood-based context, the computational cost per supervised edge depends on the size of the sampled subgraph, which in turn depends on graph structure such as degree distribution and clustering. As a result, total training compute cannot be expressed as a universal function of $(N,D)$ alone, unlike in sequence models where compute is well-approximated by $O(N \cdot D)$~\citep{hoffmann2022training}.
We therefore treat compute as a graph-conditional quantity. In addition to reporting scaling with respect to $N$ and $D$, we log the total neighborhood volume processed and wall-clock compute for each run. Any compute-optimal conclusions are interpreted as conditional on the underlying graph and sampling distribution, rather than as universal prescriptions.

\end{document}